\documentclass[10pt]{article}

\usepackage[margin=1in]{geometry}
\usepackage{natbib}
\usepackage{parskip}

\usepackage{algorithm, algorithmicx, algpseudocode}
\usepackage{amssymb}
\usepackage{amsmath}
\usepackage{graphicx}
\usepackage{mathtools}
\usepackage{needspace} 
\usepackage{multicol}
\usepackage{listings}
\usepackage{amsthm}
\usepackage[shortlabels]{enumitem}
\usepackage[italicdiff]{physics}
\usepackage{cancel}
\usepackage[dvipsnames]{xcolor}
\usepackage{verbatim}
\usepackage{comment}
\usepackage{array}
\usepackage{tikz-cd}
\usepackage{import}
\usepackage{booktabs}
\usepackage{comment}
\usepackage{array}
\usepackage[colorlinks, linkcolor=red, citecolor=blue, linktocpage=true]{hyperref}
\usepackage{forest}
\usepackage{float}
\usepackage{xspace}
\usepackage{aliascnt}
\usepackage{wrapfig}
\usepackage{thm-restate}
\usepackage{adjustbox}
\usepackage{bm}
\usepackage{bbm}
\usepackage{complexity}
\usepackage{makecell}
\usepackage{nicefrac}
\usetikzlibrary{calc, positioning, external, arrows.meta}

\newcommand{\uvec}{\boldsymbol{u}}

\usepackage{tcolorbox}
\usepackage{enumitem}
\definecolor{lightblue}{rgb}{0.8, 0.9, 1}
\definecolor{lightred}{rgb}{1, 0.8, 0.8}
\definecolor{lightgreen}{rgb}{0.8, 1, 0.8}

\tcbset{
	myblockstyle/.style={
		boxrule=0.4pt,
		colback=#1!10,
		colframe=#1!80!black,
		boxsep=2pt,
		left=-2.5pt,
		right=2pt,
		top=0pt,
		bottom=0pt,
		before skip=2pt,
		after skip=2pt,
	}
}

\newtcbox{\tcbinlinebox}[1][]{nobeforeafter,
	on line,
	colback=#1!10,
	colframe=#1!80!black,
	boxrule=0.8pt,
	arc=2pt,
	boxsep=0pt,
	left=2pt, right=2pt, top=2pt, bottom=2pt,
	valign=center
}

\usepackage{multirow}
\usepackage{colortbl}
\usepackage{makecell}

\usepackage{pifont}

\newcommand{\xmark}{\ding{55}}

\definecolor{rqBlueBack}{HTML}{EAF2FF} 
\definecolor{rqBlueFrame}{HTML}{3B82F6} 

\allowdisplaybreaks

\DeclareMathOperator*{\argmin}{argmin}

\let\cite\citep

\newfloat{protocol}{tbp}{lop}
\floatname{protocol}{Protocol}

\newtheorem{theorem}{Theorem}
\numberwithin{theorem}{section}
\newtheorem*{theorem*}{Theorem}
\newtheorem{conjecture}[theorem]{Conjecture}

\newtheorem{lemma}[theorem]{Lemma}

\theoremstyle{definition}
\newtheorem{definition}[theorem]{Definition}
\theoremstyle{condition}

\theoremstyle{assumption}
\newtheorem{assumption}[theorem]{Assumption}

\newcommand{\xvec}{\boldsymbol{x}}
\newcommand{\lvec}{\boldsymbol{\ell}}
\newcommand{\gvec}{\boldsymbol{g}}
\newcommand{\xivec}{\boldsymbol{\xi}}
\newcommand{\zetavec}{\boldsymbol{\zeta}}
\newcommand{\sigmavec}{\boldsymbol{\sigma}}

\newcommand{\nb}[3]{{\colorbox{#2}{\bfseries\sffamily\scriptsize\textcolor{white}{#1}}}{\textcolor{#2}{\sf\small\textit{#3}}}}
\newcommand{\stradi}[1]{\nb{Stradi}{red}{#1}}

\usepackage{caption}
\makeatletter
\def\thanks#1{\protected@xdef\@thanks{\@thanks
        \protect\footnotetext{#1}}}
\makeatother

\title{Truly Adapting to Adversarial Constraints in Constrained MABs}
\author{
    Francesco Emanuele Stradi \\
    \texttt{francescoemanuele.stradi@polimi.it} \\
    Politecnico di Milano 
\and
    Kalana Kalupahana \\
    \texttt{kalanakalpitha.kalupahana@mail.polimi.it} \\
    Politecnico di Milano 
\and
    Matteo Castiglioni \\
    \texttt{matteo.castiglioni@polimi.it} \\
    Politecnico di Milano 
\and
    Alberto Marchesi \\
    \texttt{alberto.marchesi@polimi.it} \\
    Politecnico di Milano 
\and
    Nicola Gatti\\
    \texttt{nicola.gatti@polimi.it} \\
    Politecnico di Milano 
}
\date{\today}

\begin{document}

\maketitle

\begin{abstract}%
  We study the constrained variant of the \emph{multi-armed bandit} (MAB) problem, in which the learner aims not only at minimizing the total loss incurred during the learning dynamic, but also at controlling the violation of multiple \emph{unknown} constraints, under both \emph{full} and \emph{bandit feedback}. We consider a non-stationary environment that subsumes both stochastic and adversarial models and where, at each round, both losses and constraints are drawn from distributions that may change arbitrarily over time. In such a setting, it is provably not possible to guarantee both sublinear regret and sublinear violation. Accordingly, prior work has mainly focused either on settings with stochastic constraints or on relaxing the benchmark with fully adversarial constraints (\emph{e.g.}, via competitive ratios with respect to the optimum). We provide the first algorithms that achieve optimal rates of regret and \emph{positive} constraint violation when the constraints are stochastic while the losses may vary arbitrarily, and that simultaneously yield guarantees that degrade smoothly with the degree of adversariality of the constraints. Specifically, under \emph{full feedback} we propose an algorithm attaining $\widetilde{\mathcal{O}}(\sqrt{T}+C)$ regret and $\widetilde{\mathcal{O}}(\sqrt{T}+C)$ {positive} violation, where $C$ quantifies the amount of non-stationarity in the constraints. We then show how to extend these guarantees when only bandit feedback is available for the losses. Finally, when \emph{bandit feedback} is available for the constraints, we design an algorithm achieving $\widetilde{\mathcal{O}}(\sqrt{T}+C)$ {positive} violation and $\widetilde{\mathcal{O}}(\sqrt{T}+C\sqrt{T})$ regret.
\end{abstract}

\section{Introduction}

Over the past few years, \emph{constrained multi-armed bandit} (MAB) problems have attracted growing attention in learning theory (see, \emph{e.g.},~\citep{liakopoulos2019cautious,pacchiano2021stochastic,Unifying_Framework}).
In the unconstrained MAB setting, the learner is evaluated solely in terms of \emph{regret}, which captures the gap between the learner's performance and that of an \emph{optimal-in-hindsight} (fixed) action.
Constrained variants introduce additional challenges: while learning, the learner must not only minimize regret, but also ensure that the prescribed constraints are not violated excessively.

In these settings, a well-known impossibility result by~\citet{Mannor} shows that it is provably impossible to achieve both regret and constraint violation that are sublinear in the number of rounds $T$ when the constraints are chosen adversarially, \emph{i.e.}, when they can change arbitrarily across rounds.\footnote{In this paper, we say that a quantity is sublinear in $T$ if it is $o(T)$.} For this reason, the literature has largely focused on different regimes. Indeed, many works study settings in which both losses and constraints are sampled i.i.d.\ from a fixed distribution at each round~\citep{Exploration_Exploitation, polytopes}, or where constraints are stochastic while losses may be adversarial~\citep{Upper_Confidence_Primal_Dual,stradi2024learning}. In both cases, optimal $\widetilde{\mathcal{O}}(\sqrt{T})$ regret and constraint violation are attainable. More recent contributions aim to provide best-of-both-worlds guarantees, allowing constraints to be either stochastic or adversarial~\citep{Unifying_Framework, stradi24a, bernasconi2024beyond}. These works typically guarantee $\widetilde{\mathcal{O}}(\sqrt{T})$ regret and violation in the stochastic regime, while in the adversarial regime they ensure sublinear violation together with sublinear $\alpha$-regret, \emph{i.e.}, regret measured against a fraction of the optimal reward. 
However, these algorithms come with several drawbacks. For instance, when constraints are only mildly adversarial, they may fail to provide any sublinear regret guarantee. Furthermore, in the adversarial regime, they typically control only a notion of violation that allows for cancellation, meaning that strictly feasible decisions can compensate for unfeasible ones. 

A recent attempt to overcome these challenges is~\citep{stradi2025taming}, where the authors study constrained MABs in a regime more general than the fully adversarial one, in which both losses and constraints are drawn from distributions that may vary across rounds.\footnote{Indeed, \citet{stradi2025taming} study \emph{constrained Markov decision processes} (CMDPs). Nonetheless, the way unknown constraints are handled is essentially equivalent in constrained MABs and CMDPs.}
%
%
Crucially, however, they assume that the non-stationarity of both losses and constraint distributions is bounded. Under this assumption, they propose a meta-algorithm achieving $\widetilde{\mathcal{O}}(\sqrt{T}+C)$ regret and \emph{positive} violation---thereby \emph{not} allowing cancellation of violation across rounds---where $C$ measures the maximum amount of non-stationarity in both losses and constraints.
Their approach has two main drawbacks. First, the resulting regret and violation bounds become vacuous (\emph{i.e.}, linear in $T$) in the adversarial-loss regime; in particular, they do not recover the optimal $\widetilde{\mathcal{O}}(\sqrt{T})$ guarantees in the mixed setting where losses are adversarial and constraints are stochastic. Second, the meta-algorithm relies on a rather complex corralling technique, \emph{i.e.}, a no-regret master procedure that selects among multiple subroutines instantiated with different values of $C$. This design makes the algorithm complex, and the analysis technically involved and somewhat intricate. Moreover, it yields a quadratic dependence on the number of constraints. Due to space constraints, we refer to Appendix~\ref{app:related} for a complete discussion on related works.

In this work, we aim to advance the theoretical understanding of constrained MABs. In particular, we address the following research question:
\begin{tcolorbox}[
colback=blue!5, 
colframe=blue!60, 
boxrule=0.8pt, 
arc=2mm, 
left=6pt,right=6pt,top=6pt,bottom=6pt
]
\begin{center}
			\emph{Is it possible to achieve sublinear regret and sublinear positive constraint violation that degrade only with the degree of non-stationarity in the constraints?}
\end{center}
\end{tcolorbox}
In this paper, we provide an affirmative answer to the question above, as we describe next.

\subsection{Original Contribution}

We study the \emph{constrained} MAB problem in which both the losses and the unknown constraints are drawn at each round from distributions that may change arbitrarily over time. We quantify the non-stationarity of the constraints via a corruption level $C$, which measures the distance between the mean values of the constraint distributions and a fictitious uncorrupted constraint vector. Crucially, our notion of corruption level accounts only for the non-stationarity of the constraints, and \emph{not} for that of the losses. Clearly, when the constraints are stochastic but stationary, $C = 0$, while in the worst case $C = \Theta(T)$ (see Section~\ref{sec:Preliminaries} for further details on the setting). Throughout the paper, we consider different feedback models. Specifically, in the full feedback case, at the end of each round the learner observes the losses and the constraint violation of every possible action. In contrast, under bandit feedback, the learner observes only the loss and the violation of the chosen action.

As a warm-up, in Section~\ref{sec:ineq} we provide a detailed discussion on the technical challenges posed by our setting. Then, we show how simple concentration results for the corrupted constraints yield $\widetilde{\mathcal{O}}(\sqrt{T}+C)$ regret $R_T$ and \emph{positive} constraint violation $V_T$ when the corruption level $C$ is known.

\paragraph{Full feedback} 
In Section~\ref{sec:full}, we focus on the full feedback setting. In this case, we propose an algorithm that achieves $\widetilde{\mathcal{O}}(\sqrt{T}+C)$ regret and \emph{positive} violation without any prior knowledge of the corruption. Our approach relies on two main components.
First, at each round, the algorithm constructs an approximate feasible decision set $\mathcal{X}_t$, by using the \emph{optimism} principle. Since $C$ is unknown, this optimistic approximation does \emph{not} ensure that an optimal strategy is included in the set at every round. To tackle this challenge, we leverage a second component: a no-regret optimization procedure that guarantees small \emph{switching regret on moving decision spaces}. Specifically, we employ online mirror descent with a fixed-share update~\citep{cesa2012mirror}, and we show that, when the number of switches of a dynamic benchmark---allowed to lie in moving decision spaces---is small, the procedure attains sublinear dynamic regret. This property enables us to effectively analyze the performance of our algorithm against a fixed benchmark that may not always belong to the per-round decision space, leading to our final result.
Finally, in Section~\ref{sec:fullconst}, we show how to extend this result to the case where only bandit feedback is available on the losses.

\paragraph{Bandit feedback} In Section~\ref{sec:bandit}, we study the bandit feedback setting. We propose an algorithm that attains $\widetilde{\mathcal{O}}(T^{\max\{\nicefrac{1}{2},\beta\}}+C)$ \emph{positive} violation and $\widetilde{\mathcal{O}}(T^\beta+CT^{1-\beta})$ regret, where $\beta$ is given as input. Specifically, by setting $\beta=\nicefrac{1}{2}$ we obtain $\widetilde{\mathcal{O}}(\sqrt{T}+C)$ \emph{positive} violation and $\widetilde{\mathcal{O}}(\sqrt{T}+C\sqrt{T})$ regret.
As we discuss in Section~\ref{sec:bandit}, under bandit feedback, we cannot rely on switching-regret guarantees to address the main technical challenges. Intuitively, in adversarial-loss settings, one cannot ensure convergence to an optimal strategy; consequently, the support of an optimal strategy is \emph{not} guaranteed to be explored sufficiently. We overcome this issue via a two-phase algorithm. In the first phase, the learner enforces uniform exploration for $\Theta(T^\beta)$ rounds. In the second phase, it runs online mirror descent~\citep{Orabona} over $\mathcal{X}_t$. The final guarantee follows from the fact that the forced exploration allows us to show that, with high probability, an approximately optimal strategy belongs to $\mathcal{X}_t$ throughout the entire second phase.

\paragraph{Comparison with the state-of-the-art} To conclude the section, we provide a brief comparison between the regret and violation bounds derived in this work and the state-of-the-art:
\begin{itemize}[noitemsep, topsep=0pt, leftmargin=*]
\item Our results match those in the \emph{stochastic constraints} literature, including settings with stochastic losses~\citep{Exploration_Exploitation}, adversarial losses with full feedback~\citep{Upper_Confidence_Primal_Dual}, and adversarial losses with bandit feedback~\citep{stradi2024learning}.
%
%
Specifically, we match their $\widetilde{O}(\sqrt{T})$ regret and violation bounds, while additionally providing guarantees when the constraints are adversarial.
\item Comparing to the \emph{best-of-both-worlds} results of~\citet{bernasconi2024beyond}, we match their guarantees when the constraints are stochastic. When the constraints are (even only mildly) adversarial, they do \emph{not} provide any sublinear regret guarantee with respect to an optimal strategy; specifically, they obtain $\widetilde{\mathcal{O}}(\sqrt{T})$ regret against a \emph{fraction} of the optimum. The \emph{same} guarantee can be easily recovered by our algorithm, as discussed in Section~\ref{sec:ineq}. In contrast, we show that our algorithm achieves sublinear regret whenever the constraints are mildly adversarial. Moreover, under \emph{adversarial} constraints,~\citet{bernasconi2024beyond} do \emph{not} provide guarantees on the \emph{positive} violation, thus allowing for cancellations. Finally, their techniques do \emph{not} extend to adversarial settings with noise, namely when both losses and constraints are drawn from distributions that vary over time.
\item Comparing to~\citet{stradi2025taming}, who provide regret and \emph{positive} violation bounds that degrade with the amount of corruption affecting \emph{both} losses and constraints, we obtain guarantees that degrade only with the degree of non-stationarity of the constraints, and are thus optimal in the setting with adversarial losses and stochastic constraints.
\end{itemize}

\section{Preliminaries}
\label{sec:Preliminaries}

We study \emph{online learning} problems~\citep{cesa2006prediction} in which a learner interacts with an unknown environment over $T$ rounds, so as to minimize long-term losses subject to $m\in\mathbb{N}_+$ \emph{unknown} constraints.
At each round $t\in[T]$,\footnote{In this paper, we denote by $[a \ldots b]$ the set of all the natural numbers from $a \in \mathbb{N}$ to $b \in \mathbb{N}$ (both included), while $[b] \coloneqq [ 1 \ldots b ]$ denotes the set of the first $b \in \mathbb{N}$ natural numbers.} the learner selects a strategy $\xvec_t\in\Delta_K$ over $K\in\mathbb{N}_+$ arms, where $\Delta_K$ is the $(K-1)$-dimensional simplex. Then, they select arm $a_t\sim \xvec_t$, suffering a loss $\ell_t(a_t)\in[0,1]$ and a constraint violation $g_{t,i}(a_t)\in[-1,1]$ for each $i\in[m]$, where non-strictly-positive values stand for satisfaction of the constraint. The loss vector $\lvec_t\in[0,1]^K$ is sampled at each round from a distribution $\mathcal{L}_t$ while, for all $i\in[m]$, the constraint vector $\gvec_{t,i}\in[-1,1]^K$ is sampled at each round from a distribution $\mathcal{G}_{t,i}$. $\mathcal{L}_t$ and $\mathcal{G}_{t,i}$ are allowed to change arbitrarily across rounds, \emph{i.e.}, they may be chosen adversarially. We refer to $\bar{\lvec}_t \in [0,1]^K$ as the expected value of $\mathcal{L}_t$, and to $\bar{\gvec}_{t,i} \in [-1,1]^K$ as the expected value of $\mathcal{G}_{t,i}$.
When \textcolor{red}{\emph{full feedback}} (on the losses and/or constraints) is available, the learner observes the full (loss and/or constraint) vectors at the end of each round. When only \textcolor{blue}{\emph{bandit feedback}} is available, the learner observes the loss and the constraint violation only of the selected arm. In Protocol~\ref{alg: Learner-Environment Interaction}, we provide the learner-environment interaction.

	\begin{protocol}[!htp]
		\caption{Learner-Environment Interaction}
		\label{alg: Learner-Environment Interaction}
        \begin{algorithmic}[1]
		\For{$t\in[T]$}
			
			\State The adversary selects $\mathcal{L}_t$, $\mathcal{G}_{t,i} $ for all $ i\in[m]$
            
             \State $\lvec_t\sim\mathcal{L}_t$ and $\gvec_{t,i}\sim\mathcal{G}_{t,i}$ for all $ i\in[m]$ are drawn
			
			\State Select strategy $\xvec_t$ and choose $a_t\sim \xvec_t$
			
			
			\State Suffer loss $\ell_t(a_t)$ and violation $g_{t,i}(a_t) $ for all $i\in[m]$
            
            \State \tcbinlinebox[red]{Observe $\lvec_t$ and $\gvec_{t,i} $ for all $ i\in[m]$} 
			
            \State \tcbinlinebox[blue]{Observe $\ell_t(a_t)$ and $g_{t,i}(a_t) $ for all $ i\in[m]$ }
			
		\EndFor
        \end{algorithmic}
	\end{protocol}

Given the impossibility of learning under adversarial constraints~\citep{Mannor}---and noting that our setting is a generalization of standard adversarial ones, where no noise is added to adversarial losses and constraints---we aim at results parametrized by the amount of adversariality of the constraints. Formally, we introduce the notion of \emph{(adversarial) corruption} of the constraints defined as $C \coloneqq \max_{i\in[m]} C_i$ where $
C_i \coloneqq \min_{\gvec_i \in[-1,1]^{K}}\sum_{t=1}^T\lVert \bar\gvec_{t,i} - \gvec_i \rVert_1$. For every $i \in [m]$, we let $    \gvec^\circ_i \in [-1,1]^K$ be the constraint vector that attains the minimum in the definition of $C_i$.

\subsection{Performance Metrics}

To define the performance metrics used to evaluate our learning algorithms, we need to introduce an \emph{offline} optimization problem.
Program~\eqref{lp:offline_opt} defines an ({offline}) \emph{optimal feasible solution in hindsight}:
\begin{equation}\label{lp:offline_opt}
	\text{OPT} \coloneqq \begin{cases}
		\min_{ \xvec\in \Delta_K} & \sum_{t=1}^T \bar{ \lvec}_{t}^{\top} \xvec \quad \text{s.t.}\\
		& \sum_{t=1}^T\bar \gvec_{t,i}^{\top}  \xvec \leq 0 \quad \forall i\in[m].
	\end{cases}
\end{equation}
Notice that, differently from standard MAB problems~\citep{Orabona}, in their constrained counterpart an optimal solution in hindsight may need to randomize over arms, since the constraints can potentially cut out vertices of the simplex, thus making choosing any arm potentially suboptimal. 

Throughout the paper, we make use of the following Slater's like assumption, which is standard in adversarial online learning with unknown constraints (see, \emph{e.g.},~\citep{Immorlica2022, Unifying_Framework,stradi24a,bernasconi2025noregret}). 
\begin{assumption}\label{cond:slater}
	There exists a strategy $\xvec^\circ\in\Delta_K$ such that $\max_{t\in[T]}\bar\gvec_{t,i}^\top \xvec^\circ < 0$ for all $i\in[m]$.
\end{assumption}
We also introduce a problem-specific \emph{feasibility parameter} $\rho\in[0,1]$ related to Assumption~\ref{cond:slater}, defined as $
\rho \coloneqq \max_{\xvec \in \Delta_K} \min_{t\in[T]}\min_{i \in [m]} -\bar \gvec_{t,i}^\top \xvec
.$
We denote by $\xvec^\diamond \in \Delta_K$ a strategy that attains the value $\rho$, which we refer to in the following as a \emph{strictly feasible strategy}.
%
Intuitively, $\rho$ represents by how much $\xvec^\diamond$ strictly satisfies the constraints.
Assumption~\ref{cond:slater} is equivalent to assuming that $\rho > 0$.

Now, we introduce the notion of \emph{(cumulative) regret} and \emph{(cumulative) positive constraint violation}, the performance metrics used to evaluate algorithms.
The regret over the $T$ rounds is defined as:
\[
R_{T} \coloneqq   \sum_{t=1}^T  \ell_t(a_t) - \text{OPT}.
\]
In the following, we denote by $\xvec^*$ a strategy solving Program~\eqref{lp:offline_opt}.
Thus, $\text{OPT}=\sum_{t=1}^T\bar \lvec_t^{\top} \xvec^*$ and the regret can be written as $R_{T} \coloneqq   \sum_{t=1}^T  \ell_t(a_t) - \sum_{t=1}^T  \bar \lvec_t^{\top}   \xvec^*$.

The cumulative positive constraint violation over $T$ rounds is:
\[
V_{T}:= \max_{i\in[m]}\sum_{t\in[T]}\big[ \bar \gvec_{t,i}^\top \xvec_t \big]^+,\] where we let $[\cdot]^+:=\max\{0, \cdot\}$.
To be consistent with the definition of $\text{OPT}$, we use the randomized strategy played by the learner instead of the actual arm (and similarly, we take the expectation over the constraint distribution) in the constraint violation definition.
Notice that replacing the strategy with the realized arm (and the expectation with the sampled constraint) in the violation definition would lead to linear violation even if the optimal solution $\xvec^*$ is played for $T$ rounds, due to the $[\cdot]^+$ operator. We remark that a sublinear bound on $V_T$ directly implies a bound on $\max_{i\in[m]}\sum_{t\in[T]} g_{t,i}(a_t)$ (where $[\cdot]^+$ is not used) up to a $\widetilde{\mathcal{O}}(\sqrt T)$ factor, thanks to a straightforward application of the Azuma–Hoeffding inequality.

The goal of the learner is to attain sublinear regret and sublinear positive constraint violation with bounds that degrade gracefully with respect to the constraint corruption term $C$.

\section{Warm-Up: The Trade-Off Between Regret and Violation}
\label{sec:ineq}

We start highlighting the technical challenges of our setting, namely, the trade-off that any algorithm has to face when dealing with an adversarial online learning problem under unknown constraints. 

It is well known that adversarial regret minimizers are capable of being no-regret with respect to any strategy that is included in their decision space $\mathcal{X}_t$ at every round $t\in[T]$ (see, \emph{e.g.}, \citep{JinLearningAdversarial2019, stradi2024learning, bernasconi2024beyond}).  Thus, the fundamental challenge of our setting is to simultaneously ensure that: (i) an optimal feasible solution is included in the decision space at every round, so that no-regret guarantees can be easily attained; and (ii) the decision spaces are \emph{not} ``too large'', which would lead to large constraint violation.

When the losses are adversarial, and each constraint $i\in [m]$ is sampled from a fixed distribution $\mathcal{G}_i$ at each round, the most natural idea is to build an optimistic feasible set by employing Hoeffding inequality. Specifically, referring to $\widehat{\gvec}_{t,i}$ as the empirical mean of the observed constraint violation, it can be easily shown that the expected value of $\mathcal{G}_i$ lies in $\widehat{\gvec}_{t,i}\pm \xivec_t$ with high probability, where $\xivec_t$ is of order $\mathcal{O}(\nicefrac{1}{\sqrt{t}})$ under full feedback, while it is of order $\mathcal{O}(\nicefrac{1}{\sqrt{N_t(a)}})$, where $N_t(a)$ is the number of pulls of arm $a$, when only bandit feedback is available. Thus, by taking at each round $(\widehat{\gvec}_{t,i} - \xivec_t)^\top \xvec \leq 0$ as an optimistic estimated constraint and optimizing over the set of strategies that satisfy it, one can show that an optimal feasible solution $\xvec^*$ is included in $\mathcal{X}_t$ at every round $t$. Similarly, by definition of $\xivec_t$, $\mathcal{X}_t$ concentrates to the true feasible decision space at a rate of $\mathcal{O}(\nicefrac{1}{\sqrt{t}})$, and this allows to show that $\mathcal{X}_t$ is \emph{not} ``too large'' and, thus, $V_T$ is sublinear.

When both the losses and the constraints are adversarial,\footnote{The same reasoning holds even when only the losses are stochastic.} and it is thus provably \emph{not} possible to simultaneously attain sublinear regret and sublinear violation, state-of-the-art results~\citep{Unifying_Framework, bernasconi2024beyond} focus on attaining sublinear (\emph{non-positive}) violation and sublinear regret with respect to a $\nicefrac{\rho}{(1+\rho)}$ fraction of the optimal \emph{reward}. 
To have a high-level intuition on how to get these results, it is sufficient to notice that Slater's condition implies that a $\nicefrac{\rho}{1+\rho}$ convex combination between the strictly feasible strategy $\xvec^\diamond$ and the optimal strategy $\xvec^*$ is feasible at every round and, thus, it is included in any ``reasonable'' per-round decision space $\mathcal{X}_t$.\footnote{By ``reasonable'' decision space, here we mean one that does not cut out strategies that are feasible at every round.} Hence, building a decision space that simply moves toward feasible strategies, based solely on the observed violations, is sufficient to attain sublinear regret w.r.t.\ a fraction $\nicefrac{\rho}{1+\rho}$ of the optimal \emph{reward}.

In this work, the goal is to attain sublinear regret. Thus, we cannot simply play strategy that are
(approximately) feasible according to observed violations, 
since we cannot cut out arms that are unfeasible only in some rounds. Indeed, notice that $\xvec^*$ may violate the constraints in many rounds. Furthermore, the use of confidence intervals is \emph{not} sufficient to guarantee that $\xvec^*$ is included in $\mathcal{X}_t$, since the constraints are corrupted by $C$, making standard confidence intervals ineffective. In the following section, as a warm-up, we show how to deal with this problem when $C$ is known.

\subsection{A Simple Approach When $C$ is Known}

Throughout this section, we assume that the value of $C$ is \emph{known} to the learner. We show that, in such a case, it is easy to build suitable decision spaces $\mathcal{X}_t$.
%
Specifically, we define an unbiased estimator of the constraint violation 
\(
\widehat g_{t,i}(a)\coloneqq \frac{1}{N_t(a)}\sum_{\tau=1}^t \mathbb{I}_\tau(a)g_\tau(a) \) for all \( a\in[K], i\in[m], t\in[T]
\), where $N_t(a)$ is the number of pulls of arm $a$ up to round $t$ and $\mathbb{I}_\tau(a)$ is the indicator function of $a$ being selected at round $\tau\in[T]$. We refer to $\widehat \gvec_{t,i}$ as the vector whose entries are the estimates $\widehat g_{t,i}(a)$ for all $a\in[K]$. We remark that, when full feedback is available, we can define $\widehat{\gvec}_{t,i}$ by using $t$ in place of $N_t(a)$ and setting $\mathbb{I}_\tau(a)=1$ for all $a \in [K]$ and $\tau \in [T]$. Now, we are able to bound the distance between the estimator and the constraint violation in hindsight, as follows.
\begin{restatable}{lemma}{lemmaconcfinal}
	\label{lem:concfinal}
	Let $\delta\in(0,1)$. When full feedback is available, with probability at least $1-\delta$ it holds:
	\begin{equation*}
		\max_{i\in[m]} \left |\widehat g_{t,i}(a) - \frac{1}{T}\sum_{t=1}^T \bar g_{t,i}(a) \right | \leq 4\sqrt{\frac{1}{t}\ln\left(\frac{TKm}{\delta}\right)} + \frac{C}{t} + \frac{C}{T} \quad \forall t\in[T], a\in[K].
	\end{equation*}
	Similarly, when only bandit feedback is available, with probability at least $1-\delta$ it holds:
	\begin{equation*}
		\max_{i\in[m]} \left |\widehat g_{t,i}(a) - \frac{1}{T}\sum_{t=1}^T \bar g_{t,i}(a) \right | \leq 4\sqrt{\frac{1}{N_t(a)}\ln\left(\frac{TKm}{\delta}\right)} + \frac{C}{N_t(a)} + \frac{C}{T} \quad \forall t\in[T], a\in[K].
	\end{equation*}
\end{restatable}
By Lemma~\ref{lem:concfinal}, when $C$ is known, one can define $\zeta_t(a)\coloneqq4\sqrt{\nicefrac{\ln({TKm}/{\delta})}{N_t(a)}}+\nicefrac{C}{N_t(a)}+ \nicefrac{C}{T}$ for every action $a \in [K]$, so that the constraint violation in hindsight belongs to the interval $\widehat \gvec_{t,i}\pm \zetavec_t$ with high probability. Thus, a simple instantiation of online mirror descent with implicit exploration~\citep{neu, JinLearningAdversarial2019} on the per-round decision space $\{\xvec\in\Delta_K: (\gvec_{t,i}-\zetavec_t)^\top \xvec\leq 0\}$ attains sublinear regret and \emph{positive} violation of order $\widetilde{\mathcal{O}}(\sqrt{T}+C)$, by a reasoning similar to the one for the stochastic setting with fixed distributions.

In the rest of the paper, we focus on overcoming the challenge that $C$ is \emph{not} known, precluding the possibility of building meaningful confidence intervals that depend on $C$. In the following, we show a substantially less trivial way to exploit the previous inequalities even when $C$ is unknown.

\section{Learning With Full Feedback: Switching Regret on Moving Decision Spaces}
\label{sec:full}
In this section, we focus on the \emph{full feedback} setting, in which the entire loss and constraint vectors are observed by the learner at the end of each round.
A simple application of Lemma~\ref{lem:concfinal} leads to: \[\max_{i\in[m]} \left |\widehat g_{t,i}(a) - \frac{1}{T}\sum_{t=1}^T \bar g_{t,i}(a) \right | \leq 4\sqrt{\frac{1}{t}\ln\left(\frac{TKm}{\delta}\right)} + \frac{C}{t} + \frac{C}{T}, \] for all \(t\in[T], a\in[K]\), which holds with probability at least $1-\delta$.
The first component of the right-hand side can be easily computed by a learning algorithm, since all the quantities are known beforehand. Instead, the term $\nicefrac{C}{t} + \nicefrac{C}{T}$ is unknown to the learner. Fortunately, the term concentrates independently of the arms pulled. As we show next, this feature is fundamental for the analysis provided in this section.  
 
\subsection{Algorithm}

In Algorithm~\ref{alg: full}, we provide the pseudocode of \texttt{ConOMD-FS}.

\begin{algorithm}[!htp]
	\caption{Constrained OMD with Fixed Share (\texttt{ConOMD-FS})}
	\label{alg: full}
	\begin{algorithmic}[1]

	\Require $T \in \mathbb{N}$, $K \in \mathbb{N}$, $\delta \in (0,1)$
	
	\State Initialize $\xvec_1 \gets \xvec_\mathcal{U}$ where $\xvec_\mathcal{U}(a) \coloneqq 1/K $ for $a\in[K]$
	
	\State Initialize $\eta \gets \sqrt{\ln(KT)/T}$

	\For{$t\in[T]$}
		
		\State Choose $a_t\sim\xvec_t$\label{alg1: line4}
		
		\State Observe loss $\lvec_t$ and violation $\gvec_{t,i} $ for all $i\in[m]$\label{alg1: line5}
		
		\State Update estimator $\widehat \gvec_{t,i} $ for all $i\in[m]$\label{alg1: line6}
		
		\State Define $\xivec_t$ as $\xi_t(a)\coloneqq4\sqrt{\frac{1}{t}\ln\left(\frac{TKm}{\delta}\right)}$ for $a\in[K]$\label{alg1: line7}
		
		\State Build $\mathcal{X}_t\coloneqq\{\xvec\in\Delta_K: (\gvec_{t,i}-\xivec_t)^\top \xvec\leq 0\}$\label{alg1: line8}
		
		\State Compute $\widetilde\xvec_{t+1}\gets \argmin_{\xvec\in\mathcal{X}_t}
		\lvec_t^\top \xvec + \frac{1}{\eta} D(\xvec||\xvec_t)$\label{alg1: line9}
		
		\State Select $\xvec_{t+1}\coloneqq (1-\frac{1}{T})\widetilde\xvec_{t+1} + \frac{1}{T}\xvec_\mathcal{U}$\label{alg1: line10}
	\EndFor
    \end{algorithmic}
\end{algorithm}

The algorithm relies on two components. First, it builds a per-round approximate feasible decision space $\mathcal{X}_t$. This is done by computing the empirical mean of the observed violation $\widehat{\gvec}_{t,i}$ for each $i \in [m]$ (Line~\ref{alg1: line6}) and by constructing a confidence interval based solely on the only term in Lemma~\ref{lem:concfinal} known to the learner, namely $\xivec_t$ (Line~\ref{alg1: line7}). Then, $\mathcal{X}_t$ is built optimistically, by taking the lower bound on the estimated constraint violation (Line~\ref{alg1: line8}). The second component is a \emph{regret minimizer} for the losses over the sets $\mathcal{X}_t$. In the following, we show that we cannot employ an arbitrary regret minimizer due to the nature of the sets $\mathcal{X}_t$, which are only approximations of the true feasible decision space and, thus, are \emph{not} guaranteed to contain $\xvec^*$ at every round. We choose \emph{online mirror descent} (OMD) with entropic regularizer~\citep{Orabona} (Line~\ref{alg1: line9}), where: $$D(\xvec_1||\xvec_2)\coloneqq \sum_{a\in[K]}x_1(a)\ln\left(\nicefrac{x_1(a)}{x_2(a)}\right)-\sum_{a\in[K]}(x_1(a)-x_2(a)),$$ and a fixed-share update~\citep{cesa2012mirror} (Line~\ref{alg1: line10}). Intuitively, the strategy selected by OMD
is combined with the uniform strategy in order to make the learning dynamic more stable.

\subsection{Analysis and Theoretical Guarantees}
In this section, we provide the theoretical guarantees of Algorithm~\ref{alg: full}, in terms of regret and violation. 

We start by providing some intuitions about the regret analysis. To do that, we first study the approximate feasible decision space $\mathcal{X}_t$ and its connection to the optimal solution $\xvec^*$. As previously discussed, since the corruption value $C$ is \emph{not} known \emph{a priori}, it is \emph{not} guaranteed that $\xvec^*\in\mathcal{X}_t$ at every round $t\in[T]$. Nonetheless, we can still show the following fundamental results. First, given Assumption~\ref{cond:slater}, it is possible to show that $\mathcal{X}_t$ is never empty. This is a consequence of both the definition of $\xvec^\diamond$ and the optimism employed in Algorithm~\ref{alg: full}. Specifically, it is possible to show that, with probability at least $1-\delta$, it holds: \[(\widehat\gvec_{t,i}-\xivec_t)^\top \xvec^\diamond\leq -\rho < 0.\]
We provide a similar result for the optimal solution, that is:
\[(\widehat\gvec_{t,i}-\xivec_t)^\top \xvec^*\leq \frac{C}{T}+\frac{C}{t},\] which follows from Lemma~\ref{lem:concfinal} with probability at least $1-\delta$.
By combining the previous results, we can always build a $t$-dependent convex combination between $\xvec^\diamond$ and $\xvec^*$, defined as:
\[\xvec^*_{\alpha_t} \coloneqq (1-\alpha_t)\xvec^\diamond + \alpha_t \xvec^*,\] which is always included in $\mathcal{X}_t$. Indeed, by setting $\alpha_t= \frac{\rho}{\rho + \nicefrac{2C}{t}}$, with probability at least $1-\delta$:
\begin{equation}
	\left(\widehat\gvec_{t,i}-\xivec_t\right)^\top \xvec^*_{\alpha_t} 
	\leq -(1-\alpha_t)\rho + \alpha_t\frac{2C}{t} \leq 0. \label{eq:convex}
\end{equation}
Thus, on the one hand, Equation~\eqref{eq:convex} shows that $\mathcal{X}_t$ is \emph{not} too far from $\xvec^*$, while, on the other hand, a general adversarial regret minimizer is guaranteed to be no-regret only against benchmarks that belong to $\mathcal{X}_t$ \emph{at every round}, which is \emph{not} the case for $\xvec^*$.
We overcome this technical challenge by showing that 
OMD with an entropic regularizer and a fixed-share update attains sublinear \emph{switching regret} of order $\widetilde{\mathcal{O}}(S\sqrt{T})$
on \emph{moving decision spaces}, that is, sublinear dynamic regret with respect to \emph{$S$-switch dynamic benchmarks on moving decision spaces}, which are formally defined as follows. 
\begin{definition}[$S$-switch dynamic benchmark]
\label{def:switch}
    Let $\{\mathcal{X}_t\}_{t=1}^T$ be the sequence of decision spaces available to the learner. Define $S\leq T$ consecutive phases over the $T$ rounds, and let $\mathcal{I}_1,\dots, \mathcal{I}_S$ be the associated partition of $[T]$. We say that $\{\uvec_t\}_{t=1}^T$ is a $S$-switch dynamic benchmark on moving decision spaces if and only if the following two conditions hold: (i) $\uvec_t=\uvec_{t^\prime}$ for all $j\in[S],$ $t,t^\prime\in\mathcal{I}_j$; and (ii) $\uvec_t \in \mathcal{X}_j$ for all $t\in[T]$, where $ \mathcal{X}_j=\bigcap_{t\in\mathcal{I}_j} \mathcal{X}_t$.
\end{definition}
Thus, the regret associated with $\{\uvec_t\}_{t=1}^T$ is defined as
\[
R_T(\{\uvec_t\}_{t=1}^T)\coloneqq \sum_{t=1}^T \left[ \lvec_t^\top \xvec_t - \lvec_t^\top \uvec_t \right],
\]
which extends the well-known notion of switching regret~\citep{cesa2012mirror} to benchmarks that may belong to different decision spaces at each round.

A key final challenge is to deal with the dependence on the number of phases $S$. Indeed, we cannot select $S=T$ (and $\uvec_t=\xvec^*_{\alpha_t}$), since it would lead to a superlinear regret bound. Nonetheless, by noticing that $\alpha_t\leq \alpha_{t+1}$ and thus, $\xvec^*_{\alpha_t}\in \mathcal{X}_\tau$ for $\tau\geq t$, we can employ a doubling trick approach.
Specifically, we define $S=\log_2(T)$ phases, so that $\underline t_1 = 1$ and $\underline t_{j+1} = 2 \underline t_j$ for all $j \in [S-1]$, where $\underline t_j \in [T]$ denotes the first round of phase $j \in [S]$. Moreover, we set $\uvec_t=\xvec^*_{\alpha_{\underline t_j}}$ for all $j \in [S], t \in \mathcal{I}_j$. This results in a logarithmic error in the final regret bound. In the following, for ease of notation and with a slight abuse of notation, we let $\alpha_j \coloneqq \alpha_{\underline t_j}$ for all $j \in [S]$.
%

We show the final regret bound in the following theorem.
\begin{restatable}{theorem}{regfull}
	\label{thm: reg_full}
	Let $\delta\in(0,1)$. With probability at least $1-3\delta$, Algorithm~\ref{alg: full} attains:
	\[R_T\leq 4  \log_2(T)\sqrt{T\ln(KT)} + \frac{2C}{\rho}\log_2(T)+ 4\sqrt{T\ln\left(\frac{TK}{\delta}\right)}.\]
\end{restatable}
Theorem~\ref{thm: reg_full} shows that Algorithm~\ref{alg: full} attains regret of order $\widetilde{\mathcal{O}}(\sqrt{T}+C)$.
Intuitively, the result is proved by employing the bound on the switching regret on moving decision spaces and the doubling trick approach. By letting $\phi: [T] \rightarrow [S]$ be a mapping such that $\phi(t)\coloneqq \sum_{j\in[S]} j \cdot \mathbb{I}\{t\in\mathcal{I}_j\}$, the final regret guarantee follows by:
\[
\sum_{t\in[T]} \left[ \lvec_t^\top\xvec^*_{\alpha_{\phi(t)}}- \lvec_t^\top\xvec^* \right] \leq \sum_{t\in[T]} \left(1-\alpha_{\phi(t)}\right) \leq \frac{2C}{\rho}\log_2(T).
\]

We are now ready to provide the violation bound. This is done in the following theorem, where we show that Algorithm~\ref{alg: full} effectively attains a positive violation of order $\widetilde{\mathcal{O}}(\sqrt{T}+C)$.
\begin{restatable}{theorem}{viofull}
	\label{thm: vio_full}
	Let $\delta\in(0,1)$. With probability at least $1-\delta$, Algorithm~\ref{alg: full} attains:
	\begin{equation*}
		V_T \leq 2+ 2C + C\ln(T) + 16\sqrt{T\ln\left(\frac{TKm}{\delta}\right)}. 
	\end{equation*}    
\end{restatable}
Intuitively, Theorem~\ref{thm: vio_full} relies on two components. First, similarly to Theorem~\ref{thm: reg_full}, the fact that $\mathcal{X}_t$ is non-empty at every round, thanks to both optimism and Assumption~\ref{cond:slater}. Thus, Algorithm~\ref{alg: full} effectively works inside $\mathcal{X}_t$, up to the fixed-share update. This leads to a negligible (constant) violation term. 
By noticing that the positive violation in hindsight is far from the uncorrupted vector $\gvec^\circ_i$ of a $C_i$ factor, it is sufficient to bound how the terms in Lemma~\ref{lem:concfinal} concentrate. Specifically, $\xivec_t$ concentrates as $1/\sqrt{t}$, leading to the $\mathcal{O}(\sqrt{T\ln({TKm}/{\delta})}$ term in the violation bound, $C/t$ concentrates as $1/t$, leading to $C\ln(T)$ violation, while we pay an additional $C$ term for the concentration of $C/T$.

\subsection{Extension to Bandit Feedback on the Losses}
\label{sec:fullconst}
In this section, we show how to extend the previous result to the case where only bandit feedback is available on the losses, \emph{i.e.}, the learner only observes $\ell_t(a_t)$ at each round $t\in[T]$.

From an algorithmic perspective, we simply modify Algorithm~\ref{alg: full} in the following way. The loss is estimated by employing an implicit exploration approach~\citep{neu}. Specifically, at each round $t \in [T]$, we build $\widehat{ \lvec}_t$ such that $\widehat{ \ell}_t(a)\coloneqq\frac{\ell_t(a)}{x_t(a)+\gamma}\mathbb{I}_t(a)$ for all $a \in [K]$, where $\gamma \coloneqq \eta/2$ is the implicit exploration factor and $\mathbb{I}_t(a)$ is the indicator function that is equal to $1$ whenever $a=a_t$. Then, OMD with fixed-share update is used on $\widehat{ \lvec}_t$ and $\eta$ is set to $\sqrt{\nicefrac{\ln(KT)}{KT}}$. Due to space constraints, we refer to Algorithm~\ref{alg: full_const} in Appendix~\ref{app:fullconst} for the complete algorithm.

We provide the regret bound attained by Algorithm~\ref{alg: full_const} in the following theorem.
\begin{restatable}{theorem}{regfullconst}
	\label{thm: reg_full_const}
	Let $\delta\in(0,1)$. With probability at least $1-6\delta$, Algorithm~\ref{alg: full_const} attains:
    \[R_T\leq K\log_2(T)\ln\left(\frac{\log_2(T)K}{\delta}\right) +11\log_2(T)\ln\left(\frac{K\log_2(T)}{\delta}\right)\sqrt{KT\ln\left(\frac{TK}{\delta}\right)} + \frac{2C}{\rho}\log_2(T).\]
\end{restatable}
Theorem~\ref{thm: reg_full_const} shows that a regret bound of order $\widetilde{\mathcal{O}}(\sqrt{T}+C)$ is still attainable when the learner has bandit feedback on the losses. The analysis of Theorem~\ref{thm: reg_full_const} shares many similarities with the full feedback case. The key difference is that we prove no-switching-regret guarantees on moving decision spaces for OMD with implicit exploration, and we employ those guarantees to attain the final regret bound.
Finally, both the bound and the analysis of the violation attained by Algorithm~\ref{alg: full_const} are equivalent to the full feedback (on the losses) case.
Thus, it is omitted for simplicity.

\section{Learning With Bandit Feedback by Forcing Exploration}
\label{sec:bandit}

In this section, we focus on the \emph{bandit feedback} setting, where the learner observes the loss and the constraint violation only for the chosen arm.

To handle bandit feedback, we require a slightly stronger assumption than Assumption~\ref{cond:slater}. Specifically, we require the existence of a strictly feasible arm, as stated in the following. 
\begin{assumption}\label{cond:slater_pure}
	There exists an arm $a^\circ\in[K]$ such that $\max_{t\in[T]}\bar g_{t,i}(a^\circ) < 0$ for all $i\in[m]$.
\end{assumption}
We remark that this assumption has already been employed in the analysis of the state-of-the-art algorithm for constrained MABs~\citep{bernasconi2024beyond} and the special case of bandits with knapsack~\citep{Immorlica2022,castiglioni2022online}.
Throughout this section, we define the problem-specific \emph{feasibility parameter} $\rho\in[0,1]$ according to Assumption~\ref{cond:slater_pure}, as follows
$
\rho \coloneqq \max_{a\in[K]} \min_{t\in[T]}\min_{i \in [m]} -\bar g_{t,i}(a)
.$
We refer to $a^\diamond \in [K]$ as the are attaining the value $\rho$.

Similarly to the full feedback setting, we employ Lemma~\ref{lem:concfinal}, which, in the bandit case, leads to: \[\max_{i\in[m]} \left |\widehat g_{t,i}(a) - \frac{1}{T}\sum_{t=1}^T \bar g_{t,i}(a) \right | \leq 4\sqrt{\frac{1}{N_t(a)}\ln\left(\frac{TKm}{\delta}\right)} + \frac{C}{N_t(a)} + \frac{C}{T}, \] for all \( t\in[T], a\in[K]\), with probability at least $1-\delta$. Notice that, while the first term can be computed by the learning algorithm, in the bandit feedback setting, the confidence interval does \emph{not} concentrate independently of the selected arms, resulting in the impossibility of applying the techniques of the previous section.

\subsection{Algorithm}

In Algorithm~\ref{alg: bandit}, we provide the pseudocode of \texttt{ExpOpt-ConOMD}. 
\begin{algorithm}[!htp]
	\caption{Explore and Optimize with Constrained OMD (\texttt{ExpOpt-ConOMD})}
	\label{alg: bandit}
	\begin{algorithmic}[1]
	
	\Require $T \in \mathbb{N}$, $K \in \mathbb{N}$, $\delta \in (0,1)$, $\beta \in [0,1]$
	
	\State Initialize $T_0 \gets K\lceil T^{\beta}\rceil$, $N_0(a) \gets 0 $ for all $a\in[K]$
	
	\State Initialize $\xvec_{T_0+1} \gets \xvec_\mathcal{U}$ where $\xvec_\mathcal{U}(a) \coloneqq 1/K $ for all $ a\in[K]$
	
	\State Initialize $\eta \gets \sqrt{\ln(KT)/KT}, \gamma \gets \eta/2$

	\For{$a\in[K]$}

		\State \label{alg3: line4}Choose $a_t=a$ for $\lceil T^{\beta}\rceil$ rounds \Comment{Exploration Phase} 
		
		\State State Observe loss $\ell_t(a_t)$ and constraint violation $g_{t,i}(a_t) $ for all $ i\in[m]$
		
		\State Update counter $N_t(a_t)$ and empirical estimator $\widehat g_{t,i}(a_t) $ for all $i\in[m]$
		
	\EndFor \label{alg3: line10}
	
	\For{$t=T_0+1,\dots, T$}

		\State Choose $a_t\sim\xvec_t$ \Comment{Optimization Phase} \label{alg3: line11}
		
		\State Observe loss $\ell_t(a_t)$ and constraint violation $g_{t,i}(a_t) $ for all $i\in[m]$
		
		\State Update counter $N_t(a_t)$ and empirical estimator $\widehat g_{t,i}(a_t) $ for all $ i\in[m]$
		
		\State Define $\xivec_t$ as $\xi_t(a)\coloneqq4\sqrt{\frac{1}{N_t(a)}\ln\left(\frac{TKm}{\delta}\right)}$ for all $a\in[K]$
		
		\State Build $\mathcal{X}_t\coloneqq\{\xvec\in\Delta_K: (\gvec_{t,i}-\xivec_t)^\top \xvec\leq 0\}$
	
		\State Build $\widehat{\lvec}_t$ such that $\widehat\ell_t(a)\coloneqq \frac{\ell_t(a)}{x_t(a)+\gamma}\mathbb{I}_t(a)$ for all $a\in[K]$
		
		\State Compute $\xvec_{t+1}\gets \argmin_{\xvec\in\mathcal{X}_t}
		\widehat\lvec_t^\top \xvec + \frac{1}{\eta} D(\xvec||\xvec_t)$
		
	\EndFor\label{alg3: line19}
	\end{algorithmic}
\end{algorithm}

The algorithm splits the learning dynamic into two phases. In the first phase, which we call \emph{exploration phase}, the learner uniformly chooses all the actions (Lines~\ref{alg3: line4}--\ref{alg3: line10}). Intuitively, this phase allows the algorithm to get a good estimate of the true feasible decision space. Notice that the length of the exploration $T_0$ is given as input through the parameter $\beta$. In the second phase, which we call \emph{optimization phase}, the algorithm employs OMD with implicit exploration over the approximate feasible sets $\mathcal{X}_t$ (Lines~\ref{alg3: line11}--\ref{alg3: line19}). The following remarks are in order. First, $\mathcal{X}_t$ is estimated given the empirical mean and the confidence intervals computed under bandit feedback. Specifically, the counters $N_t(a)$ for all actions $a\in[K]$ are employed to compute both $\widehat\gvec_{t,i}$ and $\xivec_t$. Second, Algorithm~\ref{alg: bandit} does \emph{not} employ any fixed-share update. Indeed, as we show next, we do \emph{not} require the optimization procedure to attain the sublinear switching regret property, since, when only bandit feedback is available, this property cannot be employed to bound the distance between the optimum and the best strategy inside $\mathcal{X}_t$. Thus, the fixed-share update is \emph{not} helpful in this setting.

\subsection{Analysis and Theoretical Guarantees}

In this section, we provide the theoretical guarantees of Algorithm~\ref{alg: bandit} in terms of regret and violation.

We start by providing some intuitions about the regret analysis. First, notice that the regret in the exploration phase can be easily bounded by $K(T^\beta+1)$. Thus, we can simply focus on the regret in the second phase.
By letting $\xvec^\diamond \in \Delta_K$ be the strategy that chooses arm $a^\diamond$ deterministically and $\xvec^*_{\alpha_t} \coloneqq (1-\alpha_t)\xvec^\diamond + \alpha_t \xvec^*$, it is possible to show that, with probability at least $1-\delta$:
\begin{equation}
	\left(\widehat\gvec_{t,i}-\xivec_t\right)^\top \xvec^*_{\alpha_t} \leq
    -(1-\alpha_t)\rho + \alpha_t \cdot 2C\sum_{a\in[K]}  \frac{x^*(a)}{N_t(a)}, \label{eq:convex_bandit}
\end{equation}
by following a reasoning similar to the one used for the full feedback case. Equation~\eqref{eq:convex_bandit} highlights the intrinsic difficulty of our setting when only bandit feedback is available. Specifically, to recover the same regret rate as in the full feedback case, it is necessary to make sure that the mismatch between the optimal strategy $\xvec^*$ and the pulls performed by the algorithm concentrates. 
Intuitively, this \emph{not} possible, since, when the losses are allowed to arbitrarily change across rounds, no convergence guarantees to the optimal strategy can be attained. 
We tackle this challenge by introducing the exploration phase, which forces an upper bound on such a mismatch. Specifically:
\[\left(\widehat\gvec_{t,i}-\xivec_t\right)^\top \xvec^*_{\alpha_t}  \leq -(1-\alpha_t)\rho + \alpha_t \frac{2C}{T^\beta},
\]
for all $t>T_0$. This in turn implies that the convex combination $\xvec^*_{\alpha_t}$ satisfies \((\widehat\gvec_{t,i}-\xivec_t)^\top \xvec^*_{\alpha_t}\leq 0\) and \(\xvec^*_{\alpha_{T_0}}\in\mathcal{X}_t\) for all $t>T_0$, by setting $\alpha_t=\alpha_{T_0}\coloneqq \frac{\rho}{\rho + \nicefrac{2C}{T^\beta}}$ for all $t >T_0$.
As a result, it is possible to employ the no-regret guarantees of OMD against the comparator $ \xvec^*_{\alpha_{T_0}}$, which is included in the decision space at every round of the second phase. The final result follows by bounding the distance between the optimal strategy $\xvec^*$ and $ \xvec^*_{\alpha_{T_0}}$.

We provide the regret bound attained by Algorithm~\ref{alg: bandit} in the following theorem.
\begin{restatable}{theorem}{regbandit}
	\label{thm: reg_bandit}
	Let $\delta\in(0,1)$. With probability at least $1-6\delta$, Algorithm~\ref{alg: bandit} attains:
	\[R_T\leq K (T^\beta+1) + 3 \sqrt{KT\ln(KT)}+ K\ln\left(\frac{K}{\delta}\right)+5\sqrt{T\ln\left(\frac{TK}{\delta}\right)}+\sqrt{KT}\ln\left(\frac{K}{\delta}\right) + \frac{2C}{\rho}T^{1-\beta}.\]
\end{restatable}
Theorem~\ref{thm: reg_bandit} shows that our algorithm attains a regret of order $\widetilde{\mathcal{O}}(\max\{\sqrt{T}, T^{\beta}\} + CT^{1-\beta})$, which clearly highlights the tension between the exploration and the optimization phases. Specifically, when the exploration length is large, the corruption can be mitigated by a precise estimation of the decision space; nonetheless, the additional exploration is paid in the first component of the regret bound, which encompasses the number of rounds the algorithm is not properly minimizing the loss. 

We are now ready to provide the violation bound attained by our algorithm. 
\begin{restatable}{theorem}{viobandit}
	\label{thm: vio_bandit}
	Let $\delta\in(0,1)$. With probability at least $1-3\delta$, Algorithm~\ref{alg: bandit} attains:
	\begin{equation*}
		V_T \leq 1+ K (T^\beta+1)+ C + KC + KC\ln(T) + 30\sqrt{KT\ln\left(\frac{TKm}{\delta}\right)}. 
	\end{equation*}    
\end{restatable}
Theorem~\ref{thm: vio_bandit} shows that the violation does \emph{not} get any benefit from exploration. Indeed, optimizing over $\mathcal{X}_t$ is sufficient to get the desired $\widetilde{\mathcal{O}}(\sqrt{T}+C)$ bound, while the exploration phase only adds the $K (T^\beta+1)$ term in the final result. The violation bound is proved similarly to Theorem~\ref{thm: vio_full}, with the main exception that the confidence intervals concentrate with respect to the played action only.

We remark that by setting $\beta=\nicefrac{1}{2}$---thus exploring for $K\sqrt{T}$ rounds---we get, with high probability, the following regret and violation bounds:
\[
R_T\leq \widetilde{\mathcal{O}}(\sqrt{T}+C\sqrt{T}), \ V_T\leq \widetilde{\mathcal{O}}(\sqrt{T}+C).
\]

\section{Open Problem: Towards the Lower Bound For the Bandit Case}

In this final section, we include a discussion on the technical challenges that prevented us from formally proving the lower bound for the \emph{bandit} feedback case.

First, we state the result that we conjecture may be a valid lower bound for our setting.
\begin{conjecture}
   There exist two instances of the constrained MAB problem with bandit feedback, where $C=\Theta(\sqrt{T})$, such that any algorithm attaining $\mathcal{O}(\sqrt{T})$ positive violation in one of the instances must necessarily incur $\Omega(T^{\frac{1}{2}+\alpha})$ regret in the other instance, for a constant $\alpha>0$. 
\end{conjecture}
The following considerations are in order to find proper instances. On the one hand, we recall that, if we assume that the losses have bounded and small adversariality, \emph{i.e.}, we define the corruption $C$ as the maximum between the corruption of the losses $C^\ell$ and the corruption of the constraints $C^g$, then $\widetilde{\mathcal{O}}(\sqrt{T}+C)$ regret and violation can be attained by the algorithm of~\citet{stradi2025taming}. Thus, at least in one of the instances, we need to enforce  $C^\ell=\omega(\sqrt{T})$. To have an intuition on this aspect, notice that the last term of Equation~\eqref{eq:convex_bandit} can be easily controlled in a setting with fully stochastic losses, where it is possible to converge to the optimal strategy. Similarly, when $C^\ell$ is small, it is still possible to employ uniform exploration techniques (\emph{i.e.}, uniformly explore with small probability), to reduce the impact of $C^g$ in the regret bound. On the other hand, our conjecture prescribes a corruption to be of order $\Theta(\sqrt{T})$; thus, the constraint distributions of one of the instances can be corrupted for $\Theta(\sqrt{T})$ rounds. This leads to our \emph{first challenge}, that is, the instances are in a hybrid setting between fully stochastic and adversarial ones. This prevented us from employing standard arguments such as Pinsker's inequality and KL-decomposition, which are tailored for stochastic settings, to relate the instances. Similarly, the same reasoning holds for techniques tailored for fully adversarial lower bounds---generally used for impossibility results---as we need to show that the instances are hard to distinguish due to both the noise and the corruption. 

As a second aspect, notice that, when full feedback is available to the learner, Algorithm~\ref{alg: full} shows that $\widetilde{\mathcal{O}}(\sqrt{T}+C)$ regret and violation can be attained. Thus, bandit feedback jointly with adversarial losses must play a crucial role in the lower bound. This leads to our \emph{second challenge}, that is, we need to effectively relate the corruption in one of the instances to the strategy played by the learner. In this way, we conjecture that it is possible to show that the corruption \emph{perceived} by the learner $\widetilde{C}$ is \emph{strictly larger} than the expected corruption injected in the instances by the opponent $C$, leading to the desired lower bound, after carefully noticing that the regret must scale at least as the perceived corruption. A bit more formally, we define the perceived corruption on an arm as the corruption observed on the arm normalized by the number of pulls. Intuitively, this is the quantity that really affects the estimation and the regret. For instance, suppose that we employ two instances: the first one is fully stochastic, while we inject the corruption into both losses and constraints for one action $\bar a$ of the second instance. We select the corruption on the constraint at time $t$ as $P_t^g=f(x_t(\bar a))=x_t(\bar a)$, where $x_t(\bar a)$ is the probability of playing $\bar a$. Thus, the following chain of inequalities relates the perceived corruption $\widetilde C$ to the corruption $C$:
\[\widetilde C\approx T\frac{\sum_{t\in[T]}x_t(\bar a)P_t^g}{\sum_{t\in[T]}x_t(\bar a)}=T\frac{\sum_{t\in[T]}x^2_t(\bar a)}{\sum_{t\in[T]}x_t(\bar a)} \geq T\frac{\left(\sum_{t\in[T]}x_t(\bar a)\right)^2}{T\sum_{t\in[T]}x_t(\bar a)} =\sum_{t\in[T]}x_t(\bar a) =C.\] 
The inequality holds \emph{strictly}---up to Azuma inequality concentration terms---whenever the action is \emph{not} chosen uniformly over $T$, which could be ``enforced'' by the adversariality of the losses.

\bibliographystyle{plainnat}
\bibliography{example_paper}

\newpage
\appendix
\tableofcontents

\newpage
\section{Related Works}
\label{app:related}

In this section, we provide a summary of the literature which is mainly related to our work. We first discuss the constrained online learning literature and then we survey the main results in the corruption robust online learning one.

\paragraph{Online learning under unknown constraints}

Online leaning with \emph{unknown} constraints has been recently explored (see, \emph{e.g.},~\citep{Mannor, liakopoulos2019cautious, pacchiano2021stochastic, stradi2025noregret}). In such a setting, a well known impossibility result from~\citep{Mannor} prevents any algorithm from attaining both sublinear regret and sublinear violation when the constraints are adversarial and the optimal solution is feasible in hindsight, on average. Thus, the literature mainly focused on stochastic constrained setting~\citep{chen2022strategies, pacchiano2021stochastic, polytopes, genalti2025datadependent}. Differently, some works focus on constrained online convex optimization settings (see, \emph{e.g.},~\citep{mahdavi2012trading, jenatton2016adaptive, yu2017online}). Notice that, even when full feedback is available non both losses and constraints, these works are not applicable to our setting for the two following reasons. First, we employ as a benchmark the optimum which satisfies the constraints on average, while constrained online convex optimization focus on benchmarks satisfying the constraints at each round. Second, they are not tailored to work in adversarial settings with noise. On the other hand, more recent contributions aim to provide best-of-both-worlds guarantees, allowing constraints to be either stochastic or adversarial~\citep{castiglioni2022online, Unifying_Framework, bernasconi2024beyond}. These works typically guarantee $\widetilde{\mathcal{O}}(\sqrt{T})$ regret and violation in the stochastic regime, while in the adversarial regime they ensure sublinear violation together with sublinear $\alpha$-regret, \emph{i.e.}, regret measured against a fraction of the optimal reward. 

\paragraph{Online learning in constrained Markov decision processes}
Our paper is strongly related with the constrained Markov decision processes (CMDPs)~\citep{Altman1999ConstrainedMD} literature, which we highlight in the following. \citet{Online_Learning_in_Weakly_Coupled} study episodic CMDPs with known transitions, full-information feedback, adversarial losses, and stochastic constraints. Their algorithm guarantees $\Tilde{\mathcal{O}}(\sqrt{T})$ upper bounds for both regret and constraint violations.
\citet{Constrained_Upper_Confidence} consider stochastic losses and constraints with known transition dynamics under bandit feedback. They obtain a regret bound of $\Tilde{\mathcal{O}}(T^{3/4})$ and ensure that cumulative constraint violations stay below a prescribed threshold with high probability.
\citet{bai2020provably} propose the first method achieving sublinear regret when transition probabilities are unknown, under the assumptions of deterministic rewards and structured stochastic constraints.
\citet{Exploration_Exploitation} investigate unknown stochastic transitions, rewards, and constraints in the bandit setting. They design two algorithms that achieve sublinear regret and sublinear constraint violations by carefully balancing exploration and exploitation.
\citet{Upper_Confidence_Primal_Dual} develop a primal-dual strategy inspired by \textit{optimism in the face of uncertainty}. They show that, for episodic CMDPs with adversarial losses and stochastic constraints under full-information feedback, the approach attains sublinear regret and sublinear constraint violations.
\citet{bounded} analyze stochastic rewards and constraints with sub-Gaussian noise, proving $\widetilde{O}(\sqrt{T})$ regret and zero violations when a strictly safe policy exists and is known. When such a policy is not known \emph{a priori}, the algorithm guarantees bounded violations.
\citet{ding2021provably} introduce a primal-dual, no-regret policy optimization procedure for CMDPs with stochastic rewards and constraints.
\citet{pmlr-v151-wei22a} present a model-free, simulator-free RL algorithm for CMDPs, obtaining $\widetilde{O}(T^{\nicefrac{4}{5}})$ regret with zero constraint violations, provided that the number of episodes increases exponentially in $\nicefrac{1}{\rho}$. \citet{stradioptimal} develop the first primal-dual method attaining $\widetilde{O}(\sqrt{T})$ positive constraints violation for fully stochastic CMDPs. \citet{ding_non_stationary}, and \citet{stradi2025taming} study non-stationary rewards and constraints under bounded-variation assumptions. Notice that, both the aforementioned works do not attain sublinear regret when the rewards are fully adversarial.
\citet{stradi2024learning} consider adversarial losses, stochastic constraints, and partial feedback, establishing sublinear regret together with sublinear positive constraint violations.
\citet{stradi24a} propose the first \emph{best-of-both-worlds} algorithm for CMDPs with \emph{full feedback} on rewards and constraints, and \citet{stradi2024best} extend the guarantee to the \emph{bandit} feedback scenario.

\paragraph{Corruption-robust online learning}
A related line studies \emph{corruption-robust} \emph{unconstrained} online learning, where the observed feedback is perturbed by adversarial or stochastic corruptions (e.g., budgeted or contamination-based), and the goal is to obtain regret bounds that explicitly scale with an appropriate corruption measure.
In stochastic multi-armed bandits with adversarially corrupted rewards, \citet{lykouris2018stochastic} obtain regret bounds of order
$\widetilde{\mathcal{O}}(KC \sum_{a\neq a^*} 1/\Delta_a)$, where $\Delta_a$ is the sub-optimality gap of arm $a\in[K]$.
\citet{gupta2019better} significantly sharpen this dependence getting an \emph{additive} dependence on the corruption term, namely, 
$\widetilde{\mathcal{O}}(KC + \sum_{a\neq a^*} 1/\Delta_a)$.
A complementary \emph{attacker} model, where corruptions are chosen \emph{after} observing the learner's action, is studied by \citet{yang2020adversarial_bandits_with_corruptions}, who establish regret lower bounds and separations between budget-aware and budget-agnostic strategies; see also recent developments for
stochastic bandits under attacks in \citet{wang2025stochastic_bandits_robust_to_attacks}.
Extensions beyond standard multi-armed bandits have been considered as well, \emph{e.g.}, stochastic linear bandits under adversarial attacks~\citep{bogunovic2021stochastic_linear_bandits_robust}.
More recently, stochastic contamination models allowing heavy-tailed and even unbounded corruptions have been analyzed, yielding instance-dependent lower bounds and asymptotically optimal algorithms (see, \emph{e.g.},~\citealp{basu2024bandits_corrupted_by_nature, agrawal2024crimed}).

In episodic reinforcement learning, corruption-robust guarantees for MDPs with corrupted rewards and/or transitions under bandit feedback are developed in \citet{lykouris2021corruption} and subsequently improved by \citet{chen2021improved, wei2022model}. \citet{jin2024no} study adversarial MDPs with corrupted (non-stationary) transitions providing $\widetilde{\mathcal{O}}(\sqrt{T}+C)$ regret guarantees.
Notice that, while the above contributions relate regret to corruption/non-stationarity measures, they do not address constraints satisfaction and thus do not capture the dual objective of simultaneously controlling regret and (positive) constraint violation.

\section{Omitted Proofs of Section~\ref{sec:ineq}}

\begin{lemma} 
	\label{lem:conc1}It holds:
	\[\max_{i\in[m]}\sum_{a\in[K]}\left|g^\circ_i(a)-\frac{1}{T}\sum_{t=1}^T \bar g_{t,i}(a)\right| \leq \frac{C}{T}.\] 
\end{lemma}
\begin{proof}
	By definition of the uncorrupted vector, it holds:
	\begin{align*}
		\max_{i\in[m]}\sum_{a\in[K]}\left|g^\circ_i(a)-\frac{1}{T}\sum_{t=1}^T \bar g_{t,i}(a)\right| & = \max_{i\in[m]}\sum_{a\in[K]}\left|\frac{1}{T}\sum_{t=1}^Tg^\circ_i(a)-\frac{1}{T}\sum_{t=1}^T \bar g_{t,i}(a)\right|\\
		& = \max_{i\in[m]}\frac{1}{T}\sum_{a\in[K]}\left|\sum_{t=1}^Tg^\circ_i(a)-\sum_{t=1}^{T} \bar g_{t,i}(a)\right|     \\
		& \leq  \max_{i\in[m]}\frac{1}{T}\sum_{t=1}^T\sum_{a\in[K]}\left|g^\circ_i(a)- \bar g_{t,i}(a)\right|     \\
		& = \max_{i\in[m]}\frac{C_i}{T}\\
		& = \frac{C}{T}.
	\end{align*}
	This concludes the proof.
\end{proof}

\begin{lemma}
	\label{lem:conc2}
	Let $\delta\in(0,1)$. When full feedback is available, it holds, with probability at least $1-\delta$:
	\begin{equation*}
		\max_{i\in[m]} \left |\widehat g_{t,i}(a) - g_i^\circ(a) \right | \leq 4\sqrt{\frac{1}{t}\ln\left(\frac{TKm}{\delta}\right)} + \frac{C}{t} \quad \forall t\in[T], a\in[K].
	\end{equation*}
	Similarly, when only bandit feedback is available, it holds, with probability at least $1-\delta$:
	\begin{equation*}
		\max_{i\in[m]} \left |\widehat g_{t,i}(a) - g_i^\circ(a) \right | \leq 4\sqrt{\frac{1}{N_t(a)}\ln\left(\frac{TKm}{\delta}\right)} + \frac{C}{N_t(a)} \quad \forall t\in[T], a\in[K].
	\end{equation*}
\end{lemma}
\begin{proof}
	We prove the results for the bandit feedback case. The same reasoning holds for the full feedback case, replacing $N_t(a)$ with $t$ for all $a\in[K]$. Indeed, applying the triangle inequality, it holds:
	\begin{align*}
		\max_{i\in[m]} \left |\widehat g_{t,i}(a) - g_i^\circ(a) \right | & = \max_{i\in[m]} \left |\frac{1}{N_t(a)} \sum_{\tau\in[t]}g_{\tau,i}(a)\mathbb{I}_\tau(a) \pm \frac{1}{N_t(a)} \sum_{\tau\in[t]}\bar g_{\tau,i}(a)\mathbb{I}_\tau(a) - g_i^\circ(a) \right | \\
		& \leq \max_{i\in[m]} \left |\frac{1}{N_t(a)} \sum_{\tau\in[t]}g_{\tau,i}(a)\mathbb{I}_\tau(a) - \frac{1}{N_t(a)} \sum_{\tau\in[t]}\bar g_{\tau,i}(a)\mathbb{I}_\tau(a) \right | \\ &\mkern50mu+ \max_{i\in[m]} \left | \frac{1}{N_t(a)} \sum_{\tau\in[t]}\bar g_{\tau,i}(a)\mathbb{I}_\tau(a) - g_i^\circ(a) \right |.
	\end{align*}
	
	We start bounding the first term. Let $\delta\in(0,1)$, we employ the Azuma inequality to get, with probability at least $1-\delta$:
	\begin{equation*}
		\left | \sum_{\tau\in[t]}g_{\tau,i}(a)\mathbb{I}_\tau(a) -  \sum_{\tau\in[t]}\bar g_{\tau,i}(a)\mathbb{I}_\tau(a) \right | \leq 4\sqrt{N_t(a)\ln\left(\frac{TKm}{\delta}\right)},
	\end{equation*}
	which holds for all $t\in[T],a\in[K],i\in[m]$, by union bound.
	Thus, we can conclude that:
	\begin{equation*}
		\max_{i\in[m]} \left |\frac{1}{N_t(a)} \sum_{\tau\in[t]}g_{\tau,i}(a)\mathbb{I}_\tau(a) - \frac{1}{N_t(a)} \sum_{\tau\in[t]}\bar g_{\tau,i}(a)\mathbb{I}_\tau(a) \right | = 4\sqrt{\frac{1}{N_t(a)}\ln\left(\frac{TKm}{\delta}\right)}.
	\end{equation*}
	To bound the second term, we proceed as follows:
	\begin{align*}
		\max_{i\in[m]} \left |\frac{1}{N_t(a)} \sum_{\tau\in[t]}\bar g_{\tau,i}(a)\mathbb{I}_\tau(a) - g_i^\circ(a) \right | 
		& = \max_{i\in[m]} \left |\frac{1}{N_t(a)} \sum_{\tau\in[t]}\bar g_{\tau,i}(a)\mathbb{I}_\tau(a) - \frac{1}{N_t(a)}\sum_{\tau\in[t]}g_i^\circ(a)\mathbb{I}_\tau(a) \right | \\
		& = \max_{i\in[m]} \frac{1}{N_t(a)}\left | \sum_{\tau\in[t]}\bar g_{\tau,i}(a)\mathbb{I}_\tau(a) - \sum_{\tau\in[t]}g_i^\circ(a)\mathbb{I}_\tau(a) \right | \\
		& \leq \max_{i\in[m]} \frac{1}{N_t(a)}\sum_{\tau\in[t]}\mathbb{I}_t(a)\left | \bar g_{\tau,i}(a) - g_i^\circ(a) \right |\\
		& \leq \max_{i\in[m]} \frac{C_i}{N_t(a)}\\
		& = \frac{C}{N_{t}(a)}.
	\end{align*}
	Combining the previous results concludes the proof.
\end{proof}

\lemmaconcfinal*
\begin{proof}
	Employing the triangle inequality, it holds:
	\begin{align*}
		\max_{i\in[m]} \left |\widehat g_{t,i}(a) -\frac{1}{T}\sum_{t=1}^T \bar g_{t,i}(a) \right | & = \max_{i\in[m]} \left |\widehat g_{t,i}(a)\pm g_i^\circ -\frac{1}{T}\sum_{t=1}^T \bar g_{t,i}(a) \right |\\
		& \leq \max_{i\in[m]} \left |\widehat g_{t,i}(a)- g_i^\circ \right | + \max_{i\in[m]} \left | g_i^\circ -\frac{1}{T}\sum_{t=1}^T \bar g_{t,i}(a) \right |.
	\end{align*}
	Applying Lemma~\ref{lem:conc1} to bound the second term and Lemma~\ref{lem:conc2} to bound the first term gives the result.
\end{proof}

\section{Omitted Proofs of Section~\ref{sec:full}}

\subsection{Preliminary Results}

\begin{lemma}
	\label{lem:solution}
	Let $\delta\in(0,1)$. With probability at least $1-\delta$, $\mathcal{X}_t$ is not empty at each round $t\in[T]$.
\end{lemma}
\begin{proof}
	To prove the result, we show that any strategy $\bar \xvec$ so that $\bar \gvec_{t,i}^\top \bar \xvec\leq 0$ for all $i\in[m],t\in[T]$--notice that $\bar \xvec$ exists thanks to Assumption~\ref{cond:slater}--is included with high probability in $\mathcal{X}_t$ at each round $t\in[T]$. 
	
	Let $\delta\in(0,1)$. Similarly to what was done in Lemma~\ref{lem:conc2}, we employ the Azuma inequality to get, with probability at least $1-\delta$:
	\begin{equation*}
		\left | \sum_{\tau\in[t]}g_{\tau,i}(a) -  \sum_{\tau\in[t]}\bar g_{\tau,i}(a) \right | \leq 4\sqrt{t\ln\left(\frac{TKm}{\delta}\right)},
	\end{equation*}
	which holds for all $t\in[T],a\in[K],i\in[m]$, by union bound.
	Thus, we get:
	\begin{equation*}
		\left |\frac{1}{t} \sum_{\tau\in[t]}g_{\tau,i}(a) - \frac{1}{t} \sum_{\tau\in[t]}\bar g_{\tau,i}(a) \right | = 4\sqrt{\frac{1}{t}\ln\left(\frac{TKm}{\delta}\right)} \quad \forall a\in[K], i\in[m], t\in[T].
	\end{equation*}
	From that, we get, with probability at least $1-\delta$:
	\begin{align*}
		(\widehat \gvec_{t,i}-\xivec_t)^\top \bar \xvec &= \left(\frac{1}{t} \sum_{\tau\in[t]}\gvec_{\tau,i}-\xivec_t\right)^\top \bar \xvec  \\
		& \leq \left(\frac{1}{t} \sum_{\tau\in[t]}\bar \gvec_{\tau,i}\right)^\top \bar \xvec \\
		& \leq 0,
	\end{align*}
	where the last step holds by definition of $\bar \xvec$. This concludes the proof.
\end{proof}

\subsection{Regret}

\begin{lemma}
	\label{lem:switch_full}
    Let $\xvec^*_{\phi(1)}, \dots ,\xvec^*_{\phi(T)}$ be a $S$-switch dynamic benchmark as defined in Definition~\ref{def:switch}.
	Thus, OMD with the following fixed share update:
	\[\widetilde\xvec_{t+1}\gets \argmin_{\xvec\in\mathcal{X}_t} \lvec_t^\top \xvec_t + \frac{1}{\eta} D(\xvec||\xvec_t)\ ; \
	\xvec_{t+1}\coloneqq \left(1-\frac{1}{T}\right)\widetilde\xvec_{t+1} + \frac{1}{T}\xvec_\mathcal{U},\]
	attains:
	\[R_T(\{\xvec^*_{\phi(t)}\}_{t=1}^T)
	\leq \frac{2S}{\eta} + \frac{S\ln\left(KT\right)}{\eta}+ \eta ST.\]
\end{lemma}
\begin{proof}
	To prove the result, we fix a phase $j\in[S]$, so that the baseline $\xvec_j^*$ used in the regret associated with that phase is fixed.
	
	Thus, we recall the definition of the Bregman divergence $$D(\xvec_1||\xvec_2)\coloneqq \sum_{a\in[K]}x_1(a)\ln\left(\frac{x_1(a)}{x_2(a)}\right)-\sum_{a\in[K]}(x_1(a)-x_2(a)),$$
	which is equivalent to $D(\xvec_1||\xvec_2)\coloneqq \sum_{a\in[K]}x_1(a)\ln\left(\nicefrac{x_1(a)}{x_2(a)}\right)$ whenever $\xvec_1$ and $\xvec_2$ are in the simplex
	and we define $\bar \xvec_{t+1}$ the vector whose components are $\bar x_{t+1}(a)= x_{t}(a)e^{-\eta \ell_t(a)}$ so that $$\widetilde{\xvec}_{t+1}=\argmin_{\xvec\in\mathcal{X}_t} D(\xvec||\bar {\xvec}_{t+1}).$$
	We proceed by bounding the instantaneous regret attained by the algorithm at $t\in[T]$ s.t. $\phi(t+1)=j$ as:
	\begin{align*}
		&D(\xvec_j^*\| \xvec_t) - D(\xvec_j^*\| \xvec_{t+1}) + D(\xvec_{t}\| \bar \xvec_{t+1})  \\
		&= \sum_{a\in[K]}\left[x^*_j(a)\ln\left(\frac{x^*_j(a)}{x_t(a)} \right)-x_j^*(a)+x_t(a)\right] \\
		&\quad - \sum_{a\in[K]}\left[x^*_j(a)\ln\left(\frac{x^*_j(a)}{x_{t+1}(a)} \right)-x_j^*(a)+x_{t+1}(a)\right] \\
		&\quad + \sum_{a\in[K]}\left[x_t(a)\ln\left(\frac{x_t(a)}{x_{t}(a)e^{-\eta \ell_t(a)}}\right)-x_t(a)+x_t(a)e^{-\eta\ell_t(a)}\right] \\
		&=   \sum_{a\in[K]}x^*_j(a)\ln\left(\frac{x_{t+1}(a)}{x_t(a)} \right)
		+ \eta\lvec_t^\top \xvec_t -1 + \sum_{a\in[K]}x_t(a)e^{-\eta\ell_t(a)} \\
		&= \sum_{a\in[K]}x^*_j(a)\ln\left(\frac{(1-1/T)\widetilde{x}_{t+1}(a)+ \frac{1}{T} x_{\mathcal{U}}(a)}
		{(1-1/T)\widetilde{x}_t(a)+ \frac{1}{T} x_{\mathcal{U}}(a)} \right)
		+ \eta\lvec_t^\top \xvec_t -1 + \sum_{a\in[K]}x_t(a)e^{-\eta\ell_t(a)} \\
		&\geq  \sum_{a\in[K]}x^*_j(a)\ln\left(\frac{(1-1/T)\widetilde{x}_{t+1}(a)}
		{(1-1/T)\widetilde{x}_t(a)+ \frac{1}{T} x_{\mathcal{U}}(a)} \right)
		+ \eta\lvec_t^\top \xvec_t -1 + \sum_{a\in[K]}x_t(a)e^{-\eta\ell_t(a)}
		\\
		&=  \sum_{a\in[K]}x^*_j(a)\ln\left(\frac{\widetilde{x}_{t+1}(a)}
		{(1-1/T)\widetilde{x}_t(a)+ \frac{1}{T} x_{\mathcal{U}}(a)} \right)
		- \ln\left(\frac{1}{1-\frac{1}{T}}\right)
		+ \eta\lvec_t^\top \xvec_t -1 + \sum_{a\in[K]}x_t(a)e^{-\eta\ell_t(a)}\\
		&= D(\xvec_j^*\|\xvec_t) - D(\xvec_j^*\|\widetilde{\xvec}_{t+1})
		+ \eta\lvec_t^\top \xvec_t -1 + \sum_{a\in[K]}x_t(a)e^{-\eta\ell_t(a)}
		- \ln\left(\frac{1}{1-\frac{1}{T}}\right) \\
		&\geq D(\xvec_j^*\|\xvec_t) - D(\xvec_j^*\|\bar{\xvec}_{t+1})
		+ \eta\lvec_t^\top \xvec_t -1 + \sum_{a\in[K]}x_t(a)e^{-\eta\ell_t(a)}
		- \ln\left(\frac{1}{1-\frac{1}{T}}\right) \\
		& = \sum_{a\in[K]}x^*_j(a)\ln\left(\frac{x^*_j(a)}{x_t(a)} \right) - \sum_{a\in[K]}x^*_j(a)\ln\left(\frac{x^*_j(a)}{x_{t}(a)e^{-\eta\ell_t(a)}} \right) + 1 - \sum_{a\in[K]}x_t(a)e^{-\eta\ell_t(a)} \\
		& \quad + \eta\lvec_t^\top \xvec_t -1 + \sum_{a\in[K]}x_t(a)e^{-\eta\ell_t(a)}
		- \ln\left(\frac{1}{1-\frac{1}{T}}\right) \\
		&= \eta \lvec_t^\top(\xvec_t-\xvec^*_j) - \ln\left(\frac{1}{1-\frac{1}{T}}\right),
	\end{align*}
	where we used that $D(\xvec^*_j||\bar{\xvec}_{t+1}) \geq D(\xvec^*_j||\widetilde{\xvec}_{t+1})$ by definition of $\widetilde{\xvec}_{t+1}$ and the fact that it is included in $\mathcal{X}_t$.
	Thus, we get the following bound on the instantaneous regret: 
	\[\eta \lvec_t^\top(\xvec_t-\xvec^*_j)\leq  D(\xvec_j^*|| \xvec_t) - D(\xvec_j^*|| \xvec_{t+1}) + D(\xvec_{t}|| \bar \xvec_{t+1})+ \ln\left(\frac{1}{1-\frac{1}{T}}\right).\]
	Now, we define the quantities $\mathcal{I}_j\coloneqq\{t\in[T]:\phi(t)=j\}$, $\underline t_j\coloneqq \min_{t\in\mathcal{I}_j}t$ and we proceed bounding the regret in phase $j$ as follows:
	\begin{align*}
		\eta &\sum_{t\in\mathcal{I}_j}\lvec_t^\top(\xvec_t-\xvec^*_j)\\& \leq  \sum_{t\in\mathcal{I}_j} \left[D(\xvec_j^*|| \xvec_t) - D(\xvec_j^*|| \xvec_{t+1})\right] + \sum_{t\in\mathcal{I}_j}D(\xvec_{t}|| \bar \xvec_{t+1})+ \sum_{t\in\mathcal{I}_j}\ln\left(\frac{1}{1-\frac{1}{T}}\right) \\
		& = D(\xvec_j^*|| \xvec_{\underline t_j}) - D(\xvec_j^*|| \xvec_{\underline t_{j+1}}) + \sum_{t\in\mathcal{I}_j}D(\xvec_{t}|| \bar \xvec_{t+1})+ \sum_{t\in\mathcal{I}_j}\ln\left(\frac{1}{1-\frac{1}{T}}\right)\\
		& \leq 2+  \sum_{a\in[K]}x_j^*(a)\ln\left(\frac{x_{\underline{t}_{j+1}}(a)}{x_{\underline{t}_{j}(a)}}\right) + \sum_{t\in\mathcal{I}_j}D(\xvec_{t}|| \bar \xvec_{t+1}) \\
		& \leq 2+  \ln\left(KT\right) + \sum_{t\in\mathcal{I}_j}D(\xvec_{t}|| \bar \xvec_{t+1}) \\
		& = 2+  \ln\left(KT\right) + \sum_{t\in\mathcal{I}_j}\left[\sum_{a\in[K]}x_t(a)\ln\left(\frac{x_t(a)}{x_t(a)e^{-\eta\ell_t(a)}} \right) -1 + \sum_{a\in[K]}x_t(a)e^{-\eta\ell_t(a)}\right] \\
		& = 2+ \ln\left(KT\right)+ \sum_{t\in\mathcal{I}_j}\left[\eta \sum_{a\in[K]}x_t(a)\ell_t(a)  -1 + \sum_{a\in[K]}x_t(a)e^{-\eta\ell_t(a)}\right]  \\
		& \leq 2+ \ln\left(KT\right)+ \sum_{t\in\mathcal{I}_j}\left[\eta \sum_{a\in[K]}x_t(a)\ell_t(a) -1 + 1 -\eta\sum_{a\in[K]}x_t(a)\ell_t(a)+ \sum_{a\in[K]}x_t(a)\eta^2\ell_t^2(a)\right]\\
		& \leq 2+ \ln\left(KT\right)+ \eta^2T,
	\end{align*}
	where we used $\ln\left(\frac{1}{1-\nicefrac{1}{T}}\right)\leq \frac{\nicefrac{1}{T}}{1-\nicefrac{1}{T}}=\nicefrac{1}{T-1}\leq \nicefrac{2}{T},$ the definition of the fixed share update and the inequality $e^{-z}\leq 1 - z +z^2$ for $z\geq0$.
	
	Thus, rearranging, we obtain the following bound:
	\[
	\sum_{t\in\mathcal{I}_j}\lvec_t^\top(\xvec_t-\xvec^*_j)
	\leq \frac{2}{\eta}+ \frac{\ln\left(KT\right)}{\eta}+ \eta T.
	\]
	Noticing that the regret can be rewritten as:
	\[R_T(\{\xvec^*_{\phi(t)}\}_{t=1}^T)\coloneqq \sum_{t=1}^T \left[\lvec_t^\top \xvec_t - \lvec_t^\top \xvec^*_{\phi(t)}\right]= \sum_{j\in[S]}\sum_{t\in\mathcal{I}_j}\lvec_t^\top(\xvec_t-\xvec^*_j),\]
	concludes the proof.
\end{proof}

\regfull*

\begin{proof}
	To prove the result, we first employ Lemma~\ref{lem:solution} to state that the update of Algorithm~\ref{alg: full} admits a solution for all $t\in[T]$.
	
	Thus, we study the decision space $\mathcal{X}_t$. Specifically, thanks to Lemma~\ref{lem:concfinal}, the following bound holds for all $\xvec \in \Delta_K$, $t\in[T]$ and $i\in[m]$ with probability at least $1-\delta$:
	\begin{equation*}\left(\widehat\gvec_{t,i}-\xivec_t\right)^\top \xvec \leq  \frac{1}{T}\sum_{t=1}^T\bar \gvec_{t,i}^\top \xvec + C/T + C/t,
	\end{equation*}
	which in turn implies:
	\begin{equation*}
		\left(\widehat\gvec_{t,i}-\xivec_t\right)^\top \xvec^* \leq  C/T + C/t \quad \forall t\in[T],i\in[m].
	\end{equation*}
	On the other hand, by Assumption~\ref{cond:slater} and proceeding as in Lemma~\ref{lem:solution}, it holds:
	\begin{align*}
		(\widehat \gvec_{t,i}-\xivec_t)^\top  \xvec^\diamond &= \left(\frac{1}{t} \sum_{\tau\in[t]}\gvec_{\tau,i}-\xivec_t\right)^\top  \xvec^\diamond  \\
		& \leq \left(\frac{1}{t} \sum_{\tau\in[t]}\bar \gvec_{\tau,i}\right)^\top  \xvec^\diamond \\
		& \leq -\rho,
	\end{align*}
	with probability at least $1-\delta$.
	
	Thus, we build the following convex combination between $\xvec^*$ and $\xvec^\diamond$, parametrized given $\alpha_t$, that is:
	\begin{equation*}
		\xvec^*_{\alpha_t} = (1-\alpha_t)\xvec^\diamond + \alpha_t \xvec^*.
	\end{equation*}
	Studying the per round optimistic violation attained by $\xvec^*_{\alpha_t}$, we get:
	\begin{align*}
		\left(\widehat\gvec_{t,i}-\xivec_t\right)^\top \xvec^*_{\alpha_t} &=  (1-\alpha_t)  \left(\widehat\gvec_{t,i}-\xivec_t\right)^\top \xvec^\diamond+\alpha_t\left(\widehat\gvec_{t,i}-\xivec_t\right)^\top \xvec^*\\
		& \leq -(1-\alpha_t)\rho + \alpha_t\left(\frac{C}{T}+ \frac{C}{t}\right) \\
		& \leq -(1-\alpha_t)\rho + \alpha_t\frac{2C}{t}.
	\end{align*}
	Setting $\alpha_t= \frac{\rho}{\rho + \nicefrac{2C}{t}}$, we get  $(\widehat\gvec_{t,i}-\xivec_t)^\top \xvec^*_{\alpha_t}\leq 0$, that is, $\xvec^*_{\alpha_t}\in\mathcal{X}_t$ for all $t\in[T]$.
	
	We are now ready to employ Lemma~\ref{lem:switch_full}. Notice that, since the guarantees of online mirror descent with fixed share scales linearly in the number of switches $S$, we employ a doubling trick approach. Specifically, in the analysis, the benchmark $\xvec^*_{\alpha_t}$ is updated $\log_2(T)$ times, namely, whenever the number of rounds doubles. Thus, we have the following result:
	\begin{equation*}
		R_T(\{\xvec^*_{\alpha_{\phi(t)}}\}_{t=1}^T)\coloneqq \sum_{t=1}^T \left[\lvec_t^\top \xvec_t - \lvec_t^\top \xvec^*_{\alpha_{\phi(t)}}\right]\leq \frac{2\log_2(T)}{\eta} + \frac{\log_2(T)\ln\left(KT\right)}{\eta}+ \eta T\log_2(T),
	\end{equation*}
	where a new phase $j\in[S]$ starts whenever the number of rounds has doubled.
	
	Now, we proceed by bounding the expected regret bound as follows:
	\begin{align*}
		\sum_{t\in[T]} \left[\lvec_t^\top\xvec_t - \lvec_t^\top\xvec^* \right] &= \sum_{t\in[T]} \left[\lvec_t^\top\xvec_t \pm \lvec_t^\top\xvec^*_{\alpha_{\phi(t)}}- \lvec_t^\top\xvec^* \right]\\
		& \leq \frac{2\log_2(T)}{\eta} + \frac{\log_2(T)\ln\left(KT\right)}{\eta}+ \eta T\log_2(T) + \sum_{t\in[T]} \left[ \lvec_t^\top\xvec^*_{\alpha_{\phi(t)}}- \lvec_t^\top\xvec^* \right]\\
		& \leq \frac{2\log_2(T)}{\eta} + \frac{\log_2(T)\ln\left(KT\right)}{\eta}+ \eta T\log_2(T) + \sum_{t\in[T]} \left(1-\alpha_{\phi(t)}\right) \lvec_t^\top \xvec^\diamond  \\
		& \leq \frac{2\log_2(T)}{\eta} + \frac{\log_2(T)\ln\left(KT\right)}{\eta}+ \eta T\log_2(T) + \sum_{t\in[T]} \left(1-\alpha_{\phi(t)}\right)\\
		& = \frac{2\log_2(T)}{\eta} + \frac{\log_2(T)\ln\left(KT\right)}{\eta}+ \eta T\log_2(T) + \sum_{j\in[S]} \left(1-\alpha_{j} \right) \cdot |\mathcal{I}_j|\\
		& = 
		\frac{2\log_2(T)}{\eta} + \frac{\log_2(T)\ln\left(KT\right)}{\eta}+ \eta T\log_2(T) + \sum_{j\in[S]} \frac{2C/\underline{t}_j}{\rho+2C/\underline{t}_j} \cdot |\mathcal{I}_j|\\
		& = 
		\frac{2\log_2(T)}{\eta} + \frac{\log_2(T)\ln\left(KT\right)}{\eta}+ \eta T\log_2(T) + \sum_{j\in[S]} \frac{2C}{\underline{t}_j\rho+2C} \cdot |\mathcal{I}_j|\\
		& = 
		\frac{2\log_2(T)}{\eta} + \frac{\log_2(T)\ln\left(KT\right)}{\eta}+ \eta T\log_2(T) + \sum_{i=1}^k \frac{2C}{2^{i-1}\rho+2C} \cdot 2^{i-1} \\
		& \leq \frac{2\log_2(T)}{\eta} + \frac{\log_2(T)\ln\left(KT\right)}{\eta}+ \eta T\log_2(T) + \frac{2C}{\rho}\log_2(T),
	\end{align*}
	where $|\mathcal{I}_j|$ is the number of rounds of phase $j$ and $\underline t_j\coloneqq \min_{t\in\mathcal{I}_j}t$ with $\mathcal{I}_j\coloneqq\{t\in[T]:\phi(t)=j\}$.
	
	Setting $\eta=\frac{\sqrt{\ln(KT)}}{\sqrt{T}}$ we obtain, with probability at least $1-\delta$:
	\begin{equation*}
		\sum_{t\in[T]} \left[\lvec_t^\top\xvec_t - \lvec_t^\top\xvec^* \right]\leq 4  \log_2(T)\sqrt{T\ln(KT)} + \frac{2C}{\rho}\log_2(T).
	\end{equation*}  
	Employing the Azuma-Höeffding inequality to bound $|\sum_{t\in[T]} \lvec_t^\top\xvec_t- \sum_{t\in[T]}\ell_t(a_t)|$ and $|\sum_{t\in[T]} \lvec_t^\top\xvec^*- \sum_{t\in[T]} \bar\lvec_t^\top\xvec^*|$ with an additional union bound concludes the proof.
\end{proof}

\subsection{Violation}

\viofull*

\begin{proof}
	To prove the result, we first employ Lemma~\ref{lem:solution} to state that the update of Algorithm~\ref{alg: full} admits a solution for all $t\in[T]$. Then, we bound the violation for a general constraint $i\in[m]$, which we call \[V_{T,i}\coloneqq\sum_{t\in[T]}\left[ \bar \gvec_{t,i}^\top \xvec_t \right]^+,\] for simplicity, which implies the same bound on $V_T\coloneqq\max_{i\in[m]} V_{T,i}$. 
	
	Thus, we proceed as follows:
	\begin{align*}
		V_{T,i}&=\sum_{t\in[T]}\left[ \bar \gvec_{t,i}^\top \xvec_t \right]^+\\
		& = \sum_{t\in[T]}\left[ \left(1-\frac{1}{T}\right)\bar \gvec_{t,i}^\top \widetilde{\xvec}_t + \frac{1}{T}\bar \gvec_{t,i}^\top \xvec_{\mathcal{U}}\right]^+\\
		& \leq \sum_{t\in[T]}\left[ \left(1-\frac{1}{T}\right)\bar \gvec_{t,i}^\top \widetilde{\xvec}_t\right]^+ + \sum_{t\in[T]}\left[  \frac{1}{T}\bar \gvec_{t,i}^\top \xvec_{\mathcal{U}}\right]^+ \\
		& =  \left(1-\frac{1}{T}\right)\sum_{t\in[T]}\left[ \bar \gvec_{t,i}^\top \widetilde{\xvec}_t\right]^+ + \frac{1}{T}\sum_{t\in[T]}\left[  \bar \gvec_{t,i}^\top \xvec_{\mathcal{U}}\right]^+ \\
		& \leq 1 + \sum_{t\in[T]}\left[ \bar \gvec_{t,i}^\top \widetilde{\xvec}_t\right]^+ ,
	\end{align*}
	where we simply used the definition of the fixed share update employed in Algorithm~\ref{alg: full}, the inequality $[a+b]^+\leq [a]^++[b]^+$ and the fact that the maximum violation attainable in a single round is bounded by $1$.
	
	We proceed focusing on the violations attained by $\widetilde{\xvec}_t$. Thus, it holds:
	\begin{subequations}
		\begin{align}
			\sum_{t\in[T]}\left[ \bar \gvec_{t,i}^\top \widetilde{\xvec}_t\right]^+ & =  \bar \gvec_{1,i}^\top \widetilde{\xvec}_1 + \sum_{t=2}^T\left[ \bar \gvec_{t,i}^\top \widetilde{\xvec}_t \pm  \gvec^{\circ\top}_{i} \widetilde{\xvec}_t\right]^+ \nonumber\\
			& \leq  1+  \sum_{t=2}^T\| \bar \gvec_{t,i} -  \gvec^{\circ}_{i} \|_1 +  \sum_{t=2}^T\left[   \gvec^{\circ\top}_{i} \widetilde{\xvec}_t\right]^+ \label{viofull:eq1}\\
			& \leq   1+ C +  \sum_{t=2}^T\left[   \gvec^{\circ\top}_{i} \widetilde{\xvec}_t \pm (\widehat{\gvec}_{t-1,i}-\xivec_{t-1})^\top \widetilde{\xvec}_t\right]^+ \label{viofull:eq2} \\
			& \leq   1+  C +  \sum_{t=2}^T\left[   \gvec^{\circ\top}_{i} \widetilde{\xvec}_t -  (\widehat{\gvec}_{t-1,i}-\xivec_{t-1})^\top \widetilde{\xvec}_t \right]^+ \nonumber\\ &\mkern30mu+  \sum_{t=2}^T\left[   (\widehat{\gvec}_{t-1,i}-\xivec_{t-1})^\top \widetilde{\xvec}_t \right]^+ \label{viofull:eq3} \\
			& =  1+ C +  \sum_{t=2}^T\left[   \gvec^{\circ\top}_{i} \widetilde{\xvec}_t -  (\widehat{\gvec}_{t-1,i}-\xivec_{t-1})^\top \widetilde{\xvec}_t \right]^+ \label{viofull:eq4}\\
			& \leq 1+  C +  \sum_{t=2}^T  \left[  \left( \gvec^{\circ}_{i}  -  \widehat{\gvec}_{t-1,i}\right)^\top \widetilde{\xvec}_t \right]^+ + \sum_{t=2}^T\left[ \xivec_{t-1}^\top \widetilde{\xvec}_t \right]^+ \label{viofull:eq5}\\
			& \leq  1+ C +  \sum_{t=2}^T  \left[  \sigmavec_{t-1}^\top \widetilde{\xvec}_t \right]^+ + 2 \sum_{t=2}^T\left[ \xivec_{t-1}^\top \widetilde{\xvec}_t \right]^+ \label{viofull:eq6}\\
			&  =  1+  C +  \sum_{t=2}^T   \sigmavec_{t-1}^\top \widetilde{\xvec}_t  + 2 \sum_{t=2}^T\xivec_{t-1}^\top \widetilde{\xvec}_t \label{viofull:eq7}\\
			& \leq  1+  C + C(1+\ln(T)) + 16\sqrt{T\ln\left(\frac{TKm}{\delta}\right)} \label{viofull:eq8} \\
			& = 1+ 2C + C\ln(T) + 16\sqrt{T\ln\left(\frac{TKm}{\delta}\right)} \nonumber,
		\end{align}
	\end{subequations}
	where Inequality~\eqref{viofull:eq1} holds using $[a+b]^+\leq [a]^++[b]^+$, the Hölder inequality and the fact that the violation attained in the first round is upper bounded by $1$, Inequality~\eqref{viofull:eq2} holds by definition of $C$, Inequality~\eqref{viofull:eq3} follows from $[a+b]^+\leq [a]^++[b]^+$, Equation~\eqref{viofull:eq4} holds since $   (\widehat{\gvec}_{t-1,i}-\xivec_{t-1})^\top \widetilde{\xvec}_t \leq 0$ for all $t\in[T]$ by definition of $\mathcal{X}_t$, Inequality~\eqref{viofull:eq5} follows from $[a+b]^+\leq [a]^++[b]^+$, 
	Inequality~\eqref{viofull:eq6} holds with probability at least $1-\delta$ employing Lemma~\ref{lem:conc2} and defining $\sigmavec_t\in \mathbb{R}^K$ such that $\sigma_t(a)=\nicefrac{C}{t}$ for all $a\in[K]$, Equation~\eqref{viofull:eq7} follows from the fact that the quantities inside the $[\cdot]^+$ operator are positive and Inequality~\eqref{viofull:eq8} holds since
	\[ \sum_{t=2}^T\frac{1}{t-1} \leq 1+\ln(T)
	, \quad \sum_{t=2}^T\frac{1}{\sqrt{t-1}} \leq 2\sqrt{T},
	\]
	and the fact that any strategy sums to $1$. 
	
	Combining the previous equations concludes the proof.
\end{proof}

\section{Omitted Proofs of Section~\ref{sec:fullconst}}
\label{app:fullconst}

\begin{algorithm}[!htp]
	\caption{Constrained OMD with Fixed Share and Implicit Exploration (\texttt{ConOMD-FS IX})}
	\label{alg: full_const}
    \begin{algorithmic}[1]
	
	\Require $T \in \mathbb{N}$, $K \in \mathbb{N}$, $\delta \in (0,1)$
	
	\State Initialize $\xvec_1\gets\xvec_\mathcal{U}$ where $\xvec_\mathcal{U}(a)\coloneqq1/K \ \forall a\in[K]$
	
	\State Initialize $\eta\gets\sqrt{\ln(KT)/KT}, \gamma\gets\eta/2$

	\For{$t\in[T]$}
		\State Play $a_t\sim\xvec_t$
		
		\State Observe loss $\ell_t(a_t)$ and constraints $\gvec_{t,i} \ \forall i\in[m]$
		
		\State Update empirical mean $\widehat \gvec_{t,i} \ \forall i\in[m]$
		
		\State Define $\xivec_t$ as $\xi_t(a)\coloneqq4\sqrt{\frac{1}{t}\ln\left(\frac{TKm}{\delta}\right)}$ for all $a\in[K]$
		
		\State Build $\mathcal{X}_t\coloneqq\{\xvec\in\Delta_K: (\gvec_{t,i}-\xivec_t)^\top \xvec\leq 0\}$
		
		\State Build $\widehat{\lvec}_t$ such that $\widehat\ell_t(a)\coloneqq \frac{\ell_t(a)}{x_t(a)+\gamma}\mathbb{I}_t(a)$ for all $a\in[K]$
		
		\State Compute $\widetilde\xvec_{t+1}\gets \argmin_{\xvec\in\mathcal{X}_t}
		\widehat\lvec_t^\top \xvec + \frac{1}{\eta} D(\xvec||\xvec_t)$
		
		\State Select $\xvec_{t+1}\coloneqq (1-\frac{1}{T})\widetilde\xvec_{t+1} + \frac{1}{T}\xvec_\mathcal{U}$
		
	\EndFor
    \end{algorithmic}
\end{algorithm}

\begin{lemma}
	\label{lem:switch_full_const}
	Let $\xvec^*_{\phi(1)}, \dots ,\xvec^*_{\phi(T)}$ be a $S$-switch dynamic benchmark as defined in Definition~\ref{def:switch}. Thus, OMD with implicit exploration and the following fixed share update:
	\[\widetilde\xvec_{t+1}\gets \argmin_{\xvec\in\mathcal{X}_t} \widehat \lvec_t^\top \xvec_t + \frac{1}{\eta} D(\xvec||\xvec_t) \ ; \
	\xvec_{t+1}\coloneqq \left(1-\frac{1}{T}\right)\widetilde\xvec_{t+1} + \frac{1}{T}\xvec_\mathcal{U},\]
	where $\widehat\lvec_t \text{ s.t. }  \widehat\ell_t(a)\coloneqq \frac{\ell_t(a)}{x_t(a)+\gamma}\mathbb{I}_t(a)$ for all $a\in[K]$  attains, with probability at least $1-3\delta$:
	\[R_T(\{\xvec^*_{\phi(t)}\}_{t=1}^T)
	\leq \frac{2S}{\eta}+ \frac{S\ln\left(KT\right)}{\eta}+ 2\eta SKT +  SK\ln\left(\frac{SK}{\delta}\right)+S\sqrt{2T\ln\left(\frac{S}{\delta}\right)}+\frac{S\ln\left(\frac{KS}{\delta}\right)}{\eta}.\]
\end{lemma}
\begin{proof}
	To prove the result, we fix a phase $j\in[S]$, so that the baseline $\xvec_j^*$ used in the regret associated to that phase is fixed.
	
	Thus, we recall the definition of the Bregman divergence $$D(\xvec_1||\xvec_2)\coloneqq \sum_{a\in[K]}x_1(a)\ln\left(\frac{x_1(a)}{x_2(a)}\right)-\sum_{a\in[K]}(x_1(a)-x_2(a)),$$
	which is equivalent to $D(\xvec_1||\xvec_2)\coloneqq \sum_{a\in[K]}x_1(a)\ln\left(\nicefrac{x_1(a)}{x_2(a)}\right)$ whenever $\xvec_1$ and $\xvec_2$ are in the simplex
	and we define $\bar \xvec_{t+1}$ the vector whose components are $\bar x_{t+1}(a)= x_{t}(a)e^{-\eta \widehat\ell_t(a)}$ so that $$\widetilde{\xvec}_{t+1}=\argmin_{\xvec\in\mathcal{X}_t} D(\xvec||\bar {\xvec}_{t+1}).$$
	
	We first decompose the regret in phase $j\in[S]$ as follows:
	\[
	\sum_{t\in\mathcal{I}_j}\lvec_t^\top(\xvec_t-\xvec^*_j)
	= \sum_{t\in\mathcal{I}_j}\widehat \lvec_t^\top(\xvec_t-\xvec^*_j) + \sum_{t\in\mathcal{I}_j}(\lvec_t-\widehat \lvec_t)^\top\xvec_t + \sum_{t\in\mathcal{I}_j}(\widehat\lvec_t- \lvec_t)^\top\xvec^*_j.
	\]
	We bound the three terms separately.
	
	\paragraph{Bound on the first term}
	We proceed similarly to Lemma~\ref{lem:switch_full} fixing $t\in[T]$ s.t. $\phi(t+1)=j$ and we get:
	\[\eta \widehat \lvec_t^{\top}(\xvec_t-\xvec^*_j)\leq  D(\xvec_j^*|| \xvec_t) - D(\xvec_j^*|| \xvec_{t+1}) + D(\xvec_{t}|| \bar \xvec_{t+1})+ \ln\left(\frac{1}{1-\frac{1}{T}}\right).\]
	Now, we define the quantities $\mathcal{I}_j\coloneqq\{t\in[T]:\phi(t)=j\}$, $\underline t_j\coloneqq \min_{t\in\mathcal{I}_j}t$ and we proceed bounding the first term in phase $j$ as follows:
	\begin{align*}
		\eta& \sum_{t\in\mathcal{I}_j}\widehat \lvec_t^\top(\xvec_t-\xvec^*_j) \\&\leq  \sum_{t\in\mathcal{I}_j} \left[D(\xvec_j^*|| \xvec_t) - D(\xvec_j^*|| \xvec_{t+1})\right] + \sum_{t\in\mathcal{I}_j}D(\xvec_{t}|| \bar \xvec_{t+1})+ \sum_{t\in\mathcal{I}_j}\ln\left(\frac{1}{1-\frac{1}{T}}\right) \\
		& = D(\xvec_j^*|| \xvec_{\underline t_j}) - D(\xvec_j^*|| \xvec_{\underline t_{j+1}}) + \sum_{t\in\mathcal{I}_j}D(\xvec_{t}|| \bar \xvec_{t+1})+ \sum_{t\in\mathcal{I}_j}\ln\left(\frac{1}{1-\frac{1}{T}}\right)\\
		& \leq 2+  \sum_{a\in[K]}x_j^*(a)\ln\left(\frac{x_{\underline{t}_{j+1}}(a)}{x_{\underline{t}_{j}(a)}}\right) + \sum_{t\in\mathcal{I}_j}D(\xvec_{t}|| \bar \xvec_{t+1}) \\
		& \leq 2+  \ln\left(KT\right) + \sum_{t\in\mathcal{I}_j}D(\xvec_{t}|| \bar \xvec_{t+1}) \\
		& = 2+  \ln\left(KT\right) + \sum_{t\in\mathcal{I}_j}\left[\sum_{a\in[K]}x_t(a)\ln\left(\frac{x_t(a)}{x_t(a)e^{-\eta\widehat \ell_t(a)}} \right) -1 + \sum_{a\in[K]}x_t(a)e^{-\eta\widehat \ell_t(a)}\right] \\
		& = 2+ \ln\left(KT\right)+ \sum_{t\in\mathcal{I}_j}\left[\eta \sum_{a\in[K]}x_t(a)\widehat \ell_t(a)  -1 + \sum_{a\in[K]}x_t(a)e^{-\eta\widehat \ell_t(a)}\right]  \\
		& \leq 2+ \ln\left(KT\right)+ \sum_{t\in\mathcal{I}_j}\left[\eta \sum_{a\in[K]}x_t(a)\widehat \ell_t(a) -1 + 1 -\eta\sum_{a\in[K]}x_t(a)\ell_t(a)+ \sum_{a\in[K]}x_t(a)\eta^2\widehat \ell_t(a)^2\right]\\
		& = 2+ \ln\left(KT\right)+ \eta^2 \sum_{a\in[K]}\widehat \ell_t(a)^2 x_t(a) ,
	\end{align*}
	where we used $\ln\left(\frac{1}{1-\nicefrac{1}{T}}\right)\leq \frac{\nicefrac{1}{T}}{1-\nicefrac{1}{T}}=\nicefrac{1}{T-1}\leq \nicefrac{2}{T},$ the definition of the fixed share update and the inequality $e^{-z}\leq 1 - z +z^2$ for $z\geq0$.
	
	Thus, rearranging, we obtain the following bound for the second term:
	\begin{align*}
		\sum_{t\in\mathcal{I}_j}\widehat \lvec_t^\top(\xvec_t-\xvec^*_j) &
		\leq \frac{2}{\eta}+ \frac{\ln\left(KT\right)}{\eta}+ \eta\sum_{t\in\mathcal{I}_j}\sum_{a\in[K]}\widehat \ell_t(a)^2 x_t(a)\\
		& =\frac{2}{\eta}+ \frac{\ln\left(KT\right)}{\eta}+ \eta\sum_{t\in\mathcal{I}_j}\sum_{a\in[K]} \frac{\ell_t(a)\mathbb{I}_t(a)}{x_t(a)+\gamma} \cdot\frac{\ell_t(a)\mathbb{I}_t(a)}{x_t(a)+\gamma} x_t(a) \\
		& \leq \frac{2}{\eta}+ \frac{\ln\left(KT\right)}{\eta}+ \eta\sum_{t\in\mathcal{I}_j}\sum_{a\in[K]} \frac{\ell_t(a)\mathbb{I}_t(a)}{x_t(a)+\gamma} \cdot\frac{x_t(a)}{x_t(a)+\gamma}\\ 
		& \leq \frac{2}{\eta}+ \frac{\ln\left(KT\right)}{\eta}+ \eta\sum_{t\in\mathcal{I}_j}\sum_{a\in[K]}\widehat \ell_t(a) \\
		& \leq \frac{2}{\eta}+ \frac{\ln\left(KT\right)}{\eta}+ \eta\sum_{t\in\mathcal{I}_j}\sum_{a\in[K]}\ell_t(a) + \frac{\eta K\ln\left(\frac{K}{\delta}\right)}{2\gamma} \\
		&  \leq \frac{2}{\eta}+ \frac{\ln\left(KT\right)}{\eta}+ \eta T K + \frac{\eta K\ln\left(\frac{K}{\delta}\right)}{2\gamma},
	\end{align*}
	where we employed Corollary 1 of~\citep{neu}, which holds with probability at least $1-\delta$.

	\paragraph{Bound on the second term} We bound the term of interest as follows:
	\begin{align*}
		\sum_{t\in\mathcal{I}_j}\left(\lvec_t-\widehat \lvec_t\right)^\top\xvec_t & = \sum_{t\in\mathcal{I}_j}\left(\mathbb{E}[\widehat\lvec_t]-\widehat \lvec_t\right)^\top\xvec_t + \sum_{t\in\mathcal{I}_j}\left(\lvec_t-\mathbb{E}[\widehat\lvec_t]\right)^\top\xvec_t \\
		& \leq \sqrt{2|\mathcal{I}_j|\ln\left(\frac{1}{\delta}\right)}
		+ \sum_{t\in\mathcal{I}_j}\sum_{a\in[K]}\left(\ell_t(a)-\mathbb{E}[\widehat\ell_t(a)]\right)x_t(a) \\
		& =  \sqrt{2|\mathcal{I}_j|\ln\left(\frac{1}{\delta}\right)}
		+ \sum_{t\in\mathcal{I}_j}\sum_{a\in[K]}\left(\ell_t(a)-\frac{\ell_t(a)x_t(a)}{x_t(a)+\gamma}\right)x_t(a)\\
		& \leq \sqrt{2|\mathcal{I}_j|\ln\left(\frac{1}{\delta}\right)}
		+ \gamma K |\mathcal{I}_j| \\
		& \leq \sqrt{2T\ln\left(\frac{1}{\delta}\right)}
		+ \gamma K T,
	\end{align*}
	where the first inequality holds with probability at least $1-\delta$ thanks to the Azuma-Höeffding inequality noticing that $\left |\left(\mathbb{E} [\widehat\lvec_t]-\widehat \lvec_t\right)^\top\xvec_t\right |\leq 1$ for all $t\in[T]$.
	
	\paragraph{Bound on the third term} To bound the last term we apply again Corollary 1 of~\citep{neu} to get:
	\begin{align*}
		\sum_{t\in\mathcal{I}_j}\left(\widehat\lvec_t- \lvec_t\right)^\top\xvec^*_j \leq \frac{\ln\left(\frac{K}{\delta}\right)}{2\gamma},
	\end{align*}
	which holds with probability at least $1-\delta$.
	
	\paragraph{Final bound}
	Combining the previous results, we get with probability at least $1-3\delta$ by union bound:
	\begin{equation*}
		\sum_{t\in\mathcal{I}_j}\lvec_t^\top(\xvec_t-\xvec^*_j)
		=  \frac{2}{\eta}+ \frac{\ln\left(KT\right)}{\eta}+ \eta T K + \frac{\eta K\ln\left(\frac{K}{\delta}\right)}{2\gamma}+\sqrt{2T\ln\left(\frac{1}{\delta}\right)}
		+ \gamma K T+\frac{\ln\left(\frac{K}{\delta}\right)}{2\gamma},
	\end{equation*}
	which tacking $\eta = 2\gamma$ implies:
	\begin{equation*}
		\sum_{t\in\mathcal{I}_j}\lvec_t^\top(\xvec_t-\xvec^*_j)
		=  \frac{2}{\eta}+ \frac{\ln\left(KT\right)}{\eta}+ 2\eta T K +  K\ln\left(\frac{K}{\delta}\right)+\sqrt{2T\ln\left(\frac{1}{\delta}\right)}+\frac{\ln\left(\frac{K}{\delta}\right)}{\eta},
	\end{equation*}
	Noticing that the regret can be rewritten as:
	\[R_T(\{\xvec^*_{\phi(t)}\}_{t=1}^T)\coloneqq \sum_{t=1}^T \left[\lvec_t^\top \xvec_t - \lvec_t^\top \xvec^*_{\phi(t)}\right]= \sum_{j\in[S]}\sum_{t\in\mathcal{I}_j}\lvec_t^\top(\xvec_t-\xvec^*_j),\]
	and taking a final union bound over the $S$ phases concludes the proof.    
\end{proof}

\regfullconst*
\begin{proof}
	To prove the result, we follow the analysis of Theorem~\ref{thm: reg_full}.
	
	Differently to Theorem~\ref{thm: reg_full}, we employ the regret bound attained in Lemma~\ref{lem:switch_full_const} on the benchmark $\xvec^*_{\alpha_t}$, which is updated $\log_2(T)$ times, namely, whenever the number of rounds doubles. Thus, we have the following result, which holds with probability at least $1-3\delta$:
	\begin{align*}
		R_T(\{\xvec^*_{\alpha_{\phi(t)}}\}_{t=1}^T)&\coloneqq \sum_{t=1}^T \left[\lvec_t^\top \xvec_t - \lvec_t^\top \xvec^*_{\alpha_{\phi(t)}}\right]\\ &\leq \frac{2\log_2(T)}{\eta}+ \frac{\log_2(T)\ln\left(KT\right)}{\eta}+ 2\eta KT\log_2(T) \\ & \quad +  K\log_2(T)\ln\left(\frac{\log_2(T)K}{\delta}\right)+\log_2(T)\sqrt{2T\ln\left(\frac{\log_2(T)}{\delta}\right)}\\ &\quad +\frac{\log_2(T)\ln\left(\frac{K\log_2(T)}{\delta}\right)}{\eta}.
	\end{align*}
	Following the analysis of Theorem~\ref{thm: reg_full}, we get:
	\begin{align*}
		\sum_{t\in[T]} \left[\lvec_t^\top\xvec_t - \lvec_t^\top\xvec^* \right] &= \sum_{t\in[T]} \left[\lvec_t^\top\xvec_t \pm \lvec_t^\top\xvec^*_{\alpha_{\phi(t)}}- \lvec_t^\top\xvec^* \right]\\
		& \leq \frac{2\log_2(T)}{\eta}+ \frac{\log_2(T)\ln\left(KT\right)}{\eta}+ 2\eta KT\log_2(T) \\ & \quad +  K\log_2(T)\ln\left(\frac{\log_2(T)K}{\delta}\right)+\log_2(T)\sqrt{2T\ln\left(\frac{\log_2(T)}{\delta}\right)}\\ &\quad +\frac{\log_2(T)\ln\left(\frac{K\log_2(T)}{\delta}\right)}{\eta} + \frac{2C}{\rho}\log_2(T),
	\end{align*}
	which holds with probability at least $1-4\delta$, by union bound.
	
	Setting $\eta=\frac{\sqrt{\ln(KT)}}{\sqrt{KT}}$ we obtain, with probability at least $1-4\delta$:
	\begin{equation*}
		\sum_{t\in[T]} \left[\lvec_t^\top\xvec_t - \lvec_t^\top\xvec^* \right] \leq    K\log_2(T)\ln\left(\frac{\log_2(T)K}{\delta}\right) +7\log_2(T)\ln\left(\frac{K\log_2(T)}{\delta}\right)\sqrt{KT\ln(KT)} + \frac{2C}{\rho}\log_2(T).
	\end{equation*}  
	Employing the Azuma-Höeffding inequality to bound $|\sum_{t\in[T]} \lvec_t^\top\xvec_t- \sum_{t\in[T]}\ell_t(a_t)|$ and $|\sum_{t\in[T]} \lvec_t^\top\xvec^*- \sum_{t\in[T]} \bar\lvec_t^\top\xvec^*|$ with an additional union bound concludes the proof.
\end{proof}

\section{Omitted Proofs of Section~\ref{sec:bandit}}

\subsection{Preliminary Results}

\begin{lemma}
	\label{lem:solution_bandit}
	Let $\delta\in(0,1)$. With probability at least $1-\delta$, $\mathcal{X}_t$ is not empty at each round $t\in[T]$.
\end{lemma}
\begin{proof}
	To prove the result, we show that any action $\bar a$ so that $\bar g_{t,i} (\bar a)\leq 0$ for all $i\in[m],t\in[T]$--notice that $\bar a$ exists thanks to Assumption~\ref{cond:slater_pure}--is included with high probability in $\mathcal{X}_t$ at each round $t\in[T]$. 
	
	Let $\delta\in(0,1)$. Similarly to what was done in Lemma~\ref{lem:conc2}, we employ the Azuma inequality to get, with probability at least $1-\delta$:
	\begin{equation*}
		\left | \sum_{\tau\in[t]}g_{\tau,i}(a)\mathbb{I}_\tau(a) -  \sum_{\tau\in[t]}\bar g_{\tau,i}(a)\mathbb{I}_\tau(a) \right | \leq 4\sqrt{N_t(a)\ln\left(\frac{TKm}{\delta}\right)},
	\end{equation*}
	which holds for all $t\in[T],a\in[K],i\in[m]$, by union bound.
	Thus, we get:
	\begin{equation*}
		\left |\frac{1}{N_t(a)} \sum_{\tau\in[t]}g_{\tau,i}(a)\mathbb{I}_\tau(a) - \frac{1}{N_t(a)} \sum_{\tau\in[t]}\bar g_{\tau,i}(a)\mathbb{I}_\tau(a) \right | = 4\sqrt{\frac{1}{N_t(a)}\ln\left(\frac{TKm}{\delta}\right)},
	\end{equation*}
	for all $a\in[K], i\in[m], t\in[T]$.
	
	From that, we get, with probability at least $1-\delta$:
	\begin{align*}
		\widehat g_{t,i}(\bar a)-\xi_t(\bar a) &= \frac{1}{N_t(\bar a)} \sum_{\tau\in[t]}g_{\tau,i}(\bar a)\mathbb{I}_\tau(\bar a)-\xi_t(\bar a)  \\
		& \leq\frac{1}{N_t(\bar a)} \sum_{\tau\in[t]}\bar g_{\tau,i}(\bar a)\mathbb{I}_\tau(\bar a) \\
		& \leq 0,
	\end{align*}
	where the last step holds by definition of $\bar a$. This concludes the proof.
\end{proof}

\subsection{Regret}

\regbandit*

\begin{proof}
	To prove the result, we first employ Lemma~\ref{lem:solution_bandit} to state that the update of Algorithm~\ref{alg: bandit} admits a solution for all $t\in[T]$.
	
	Then, we split the regret as follows:
	\begin{align*}
		R_T&\coloneqq\sum_{t=1}^T  \ell_t(a_t) - \sum_{t=1}^T  \bar \lvec_t^{\top}   \xvec^*\\
		& = \sum_{t\in[T_0]}  \ell_t(a_t) - \sum_{t\in[T_0]}  \bar \lvec_t^{\top}   \xvec^* + \sum_{t=T_0+1}^T  \ell_t(a_t) - \sum_{t=T_0+1}^T  \bar \lvec_t^{\top}   \xvec^*\\
		& \leq K(T^\beta +1) + \sum_{t=T_0+1}^T  \ell_t(a_t) - \sum_{t=T_0+1}^T  \bar \lvec_t^{\top}   \xvec^*.
	\end{align*}
	Thus, we study the decision space $\mathcal{X}_t$. Specifically, thanks to Lemma~\ref{lem:concfinal}, the following bound holds for all $\xvec \in \Delta_K$, $t\in[T]$ and $i\in[m]$ with probability at least $1-\delta$:
	\begin{equation*}\left(\widehat\gvec_{t,i}-\xivec_t\right)^\top \xvec \leq  \frac{1}{T}\sum_{t=1}^T\bar \gvec_{t,i}^\top \xvec + C\sum_{a\in[K]}  \frac{x(a)}{N_t(a)} +  \frac{C}{T} ,
	\end{equation*}
	which in turn implies:
	\begin{equation*}
		\left(\widehat\gvec_{t,i}-\xivec_t\right)^\top \xvec^* \leq  C\sum_{a\in[K]}  \frac{x^*(a)}{N_t(a)} +  \frac{C}{T}\quad \forall t\in[T],i\in[m].
	\end{equation*}
	On the other hand, by Assumption~\ref{cond:slater_pure} and proceeding as in Lemma~\ref{lem:solution_bandit}, it holds:
	\begin{align*}
		\widehat g_{t,i}(a^\diamond)-\xi_t(a^\diamond) &= \frac{1}{N_t(a^\diamond)} \sum_{\tau\in[t]}g_{\tau,i}(a^\diamond)\mathbb{I}_\tau(a^\diamond)-\xi_t(a^\diamond)  \\
		& \leq\frac{1}{N_t(a^\diamond)} \sum_{\tau\in[t]}\bar g_{\tau,i}(a^\diamond)\mathbb{I}_\tau(a^\diamond) \\
		& \leq -\rho,
	\end{align*}
	with probability at least $1-\delta$.
	
	Thus, we build the following convex combination between $\xvec^*$ and $\xvec^\diamond$---where, in this case, we define $\xvec^\diamond$ as the strategy the play $a^\diamond$ deterministically---parametrized given $\alpha_t$, that is:
	\begin{equation*}
		\xvec^*_{\alpha_t} = (1-\alpha_t)\xvec^\diamond + \alpha_t \xvec^*.
	\end{equation*}
	Studying the per round optimistic violation attained by $\xvec^*_{\alpha_t}$, we get:
	\begin{align*}
		\left(\widehat\gvec_{t,i}-\xivec_t\right)^\top \xvec^*_{\alpha_t} &=  (1-\alpha_t)  \left(\widehat\gvec_{t,i}-\xivec_t\right)^\top \xvec^\diamond+\alpha_t\left(\widehat\gvec_{t,i}-\xivec_t\right)^\top \xvec^*\\
		& \leq -(1-\alpha_t)\rho + \alpha_t\left(C\sum_{a\in[K]}  \frac{x^*(a)}{N_t(a)} +  \frac{C}{T}\right) \\
		& \leq -(1-\alpha_t)\rho + \alpha_t \cdot 2C\sum_{a\in[K]}  \frac{x^*(a)}{N_t(a)} .
	\end{align*}
	Now, we focus on the case $t>T_0$, that is, Algorithm~\ref{alg: bandit} has concluded the forced exploration phase. In such a case, the algorithm guarantees $N_t(a)\geq T^\beta$ for all $a\in[K]$. From that, we have:
	\[\left(\widehat\gvec_{t,i}-\xivec_t\right)^\top \xvec^*_{\alpha_t}  \leq -(1-\alpha_t)\rho + \alpha_t \frac{2C}{T^\beta},
	\]
	for all $t>T_0$.
	Setting $\alpha_t=\alpha_{T_0}\coloneqq \frac{\rho}{\rho + \nicefrac{2C}{T^\beta}}$, we get  $(\widehat\gvec_{t,i}-\xivec_t)^\top \xvec^*_{\alpha_t}\leq 0$, that is, $\xvec^*_{\alpha_{T_0}}\in\mathcal{X}_t$ for all $t>T_0$.
	We are now ready to employ the guarantees of online mirror descent with negative entropy regularizer. Indeed, by standard analysis (e.g., following the analysis Lemma~\ref{lem:switch_full_const} without the fixed share update) and setting $\gamma=\eta/2$, we have the following result for all $\xvec\in \cap_{
		t=T_{0}+1}^T\mathcal{X}_t$:
	\begin{equation*}
		\sum_{t=T_0+1}^T \left[\lvec_t^\top \xvec_t - \lvec_t^\top \xvec\right]\leq \frac{\ln\left(KT\right)}{\eta}+ 2\eta T K +  K\ln\left(\frac{K}{\delta}\right)+\sqrt{2T\ln\left(\frac{1}{\delta}\right)}+\frac{\ln\left(\frac{K}{\delta}\right)}{\eta}, 
	\end{equation*}
	which holds with probability at least $1-3\delta$.
	
	Now, we are ready to bound the expected regret bound as follows:
	\begin{align*}
		\sum_{t=T_0+1}^T& \left[\lvec_t^\top\xvec_t - \lvec_t^\top\xvec^* \right] \\&= \sum_{t=T_0+1}^T \left[\lvec_t^\top\xvec_t \pm \lvec_t^\top\xvec^*_{\alpha_{T_0}}- \lvec_t^\top\xvec^* \right]\\
		& \leq \frac{\ln\left(KT\right)}{\eta}+ 2\eta T K +  K\ln\left(\frac{K}{\delta}\right)+\sqrt{2T\ln\left(\frac{1}{\delta}\right)}+\frac{\ln\left(\frac{K}{\delta}\right)}{\eta} + \sum_{t=T_0+1}^T \left[ \lvec_t^\top\xvec^*_{\alpha_{T_0}}- \lvec_t^\top\xvec^* \right]\\
		& \leq \frac{\ln\left(KT\right)}{\eta}+ 2\eta T K +  K\ln\left(\frac{K}{\delta}\right)+\sqrt{2T\ln\left(\frac{1}{\delta}\right)}+\frac{\ln\left(\frac{K}{\delta}\right)}{\eta} + \sum_{t=T_0+1}^T \left(1-\alpha_{T_0}\right) \lvec_t^\top \xvec^\diamond  \\
		& \leq \frac{\ln\left(KT\right)}{\eta}+ 2\eta T K +  K\ln\left(\frac{K}{\delta}\right)+\sqrt{2T\ln\left(\frac{1}{\delta}\right)}+\frac{\ln\left(\frac{K}{\delta}\right)}{\eta} + \sum_{t=T_0+1}^T \left(1-\alpha_{T_0}\right)\\
		& \leq \frac{\ln\left(KT\right)}{\eta}+ 2\eta T K +  K\ln\left(\frac{K}{\delta}\right)+\sqrt{2T\ln\left(\frac{1}{\delta}\right)}+\frac{\ln\left(\frac{K}{\delta}\right)}{\eta} +  \left(1-\alpha_{T_0} \right) \cdot T\\
		& = \frac{\ln\left(KT\right)}{\eta}+ 2\eta T K +  K\ln\left(\frac{K}{\delta}\right)+\sqrt{2T\ln\left(\frac{1}{\delta}\right)}+\frac{\ln\left(\frac{K}{\delta}\right)}{\eta} +  \frac{2CT}{\rho T^\beta + 2C}\\
		& \leq \frac{\ln\left(KT\right)}{\eta}+ 2\eta T K +  K\ln\left(\frac{K}{\delta}\right)+\sqrt{2T\ln\left(\frac{1}{\delta}\right)}+\frac{\ln\left(\frac{K}{\delta}\right)}{\eta} +  \frac{2C}{\rho} T^{1-\beta},
	\end{align*}
	which holds with probability at least $1-4\delta$ by union bound.
	
	Setting $\eta=\frac{\sqrt{\ln(KT)}}{\sqrt{KT}}$ we obtain, with probability at least $1-4\delta$:
	\begin{equation*}
		\sum_{t=T_0+1}^T \left[\lvec_t^\top\xvec_t - \lvec_t^\top\xvec^* \right]\leq 3 \sqrt{KT\ln(KT)}+ K\ln\left(\frac{K}{\delta}\right)+\sqrt{2T\ln\left(\frac{1}{\delta}\right)}+\sqrt{KT}\ln\left(\frac{K}{\delta}\right) + \frac{2C}{\rho}T^{1-\beta}.
	\end{equation*}  
	Employing the Azuma-Höeffding inequality to bound $|\sum_{t=T_0+1}^T \lvec_t^\top\xvec_t- \sum_{t=T_0+1}^T\ell_t(a_t)|$ and $\sum_{t=T_0+1}^T\lvec_t^\top\xvec^*- \sum_{t=T_0+1}^T \bar\lvec_t^\top\xvec^*|$ with an additional union bound and adding the regret incurred in the exploration phase concludes the proof.
\end{proof}

\subsection{Violation}
\viobandit*

\begin{proof}
	To prove the result, we first employ Lemma~\ref{lem:solution_bandit} to state that the update of Algorithm~\ref{alg: bandit} admits a solution for all $t\in[T]$. Then, we bound the violation for a general constraint $i\in[m]$, which we call \[V_{T,i}\coloneqq\sum_{t\in[T]}\left[ \bar \gvec_{t,i}^\top \xvec_t \right]^+,\] for simplicity, which implies the same bound on $V_T\coloneqq\max_{i\in[m]} V_{T,i}$. 
	
	Thus, we proceed splitting the violation as follows:
	\begin{align*}
		V_{T,i}&=\sum_{t\in[T]}\left[ \bar \gvec_{t,i}^\top \xvec_t \right]^+\\
		& \leq K (T^\beta+1) + \sum_{t=T_0+1}^T\left[ \bar \gvec_{t,i}^\top \xvec_t \right]^+, 
	\end{align*}
	where we simply used the fact that  the maximum violation attainable in a single round is bounded by $1$.
	
	We proceed bounding the violation attained in the second phase of the algorithm. Thus, it holds:
	\begin{subequations}
		\begin{align}
			\sum_{t=T_0+1}^T\left[ \bar \gvec_{t,i}^\top \xvec_t\right]^+ & =  \bar \gvec_{T_0+1,i}^\top \xvec_{T_0+1} + \sum_{t=T_0+2}^T\left[ \bar \gvec_{t,i}^\top \xvec_t \pm  \gvec^{\circ\top}_{i} \xvec_t\right]^+ \nonumber\\
			& \leq  1+  \sum_{t=T_0+2}^T\| \bar \gvec_{t,i} -  \gvec^{\circ}_{i} \|_1 +  \sum_{t=T_0+2}^T\left[   \gvec^{\circ\top}_{i} \xvec_t\right]^+ \label{viobandit:eq1}\\
			& \leq   1+ C +  \sum_{t=T_0+2}^T\left[   \gvec^{\circ\top}_{i} \xvec_t \pm (\widehat{\gvec}_{t-1,i}-\xivec_{t-1})^\top \xvec_t\right]^+ \label{viobandit:eq2} \\
			& \leq   1+  C +  \sum_{t=T_0+2}^T\left[   \gvec^{\circ\top}_{i} \xvec_t -  (\widehat{\gvec}_{t-1,i}-\xivec_{t-1})^\top \xvec_t \right]^+  \nonumber\\ &\mkern30mu+  \sum_{t=T_0+2}^T\left[   (\widehat{\gvec}_{t-1,i}-\xivec_{t-1})^\top \xvec_t \right]^+ \label{viobandit:eq3} \\
			& =  1+ C +  \sum_{t=T_0+2}^T\left[   \gvec^{\circ\top}_{i} \xvec_t -  (\widehat{\gvec}_{t-1,i}-\xivec_{t-1})^\top \xvec_t \right]^+ \label{viobandit:eq4}\\
			& \leq 1+  C +  \sum_{t=T_0+2}^T  \left[  \left( \gvec^{\circ}_{i}  -  \widehat{\gvec}_{t-1,i}\right)^\top \xvec_t \right]^+ + \sum_{t=T_0+2}^T\left[ \xivec_{t-1}^\top \xvec_t \right]^+ \label{viobandit:eq5}\\
			& \leq  1+ C +  \sum_{t=T_0+2}^T  \left[  \sigmavec_{t-1}^\top \xvec_t \right]^+ + 2 \sum_{t=T_0+2}^T\left[ \xivec_{t-1}^\top \xvec_t \right]^+ \label{viobandit:eq6}\\
			&  =  1+  C +  \sum_{t=T_0+2}^T   \sigmavec_{t-1}^\top \xvec_t  + 2 \sum_{t=T_0+2}^T\xivec_{t-1}^\top \xvec_t \label{viobandit:eq7}\\
			&  \leq   1+  C +  \sum_{t=T_0+2}^T \sum_{a\in[K]}  \sigma_{t-1}(a) \mathbb{I}_t(a)  \nonumber\\ &\mkern30mu+ 2 \sum_{t=T_0+2}^T \sum_{a\in[K]} \xi_{t-1}(a)\mathbb{I}_t(a) + 6\sqrt{T\ln\left(\frac{TK}{\delta}\right)} \label{viobandit:eq8}\\
			& \leq  1+  C + CK(1+\ln(T)) + 24\sqrt{KT\ln\left(\frac{TKm}{\delta}\right)} + 6\sqrt{T\ln\left(\frac{TK}{\delta}\right)}\label{viobandit:eq9} \\
			& \leq 1+ C + KC + KC\ln(T) + 30\sqrt{KT\ln\left(\frac{TKm}{\delta}\right)}  \nonumber,
		\end{align}
	\end{subequations}
	where Inequality~\eqref{viobandit:eq1} holds using $[a+b]^+\leq [a]^++[b]^+$, the Hölder inequality and the fact that the violation for each round is upper bounded by $1$, Inequality~\eqref{viobandit:eq2} holds by definition of $C$, Inequality~\eqref{viobandit:eq3} follows from $[a+b]^+\leq [a]^++[b]^+$, Equation~\eqref{viobandit:eq4} holds since $   (\widehat{\gvec}_{t-1,i}-\xivec_{t-1})^\top \xvec_t \leq 0$ for all $t\in[T]$ by definition of $\mathcal{X}_t$, Inequality~\eqref{viobandit:eq5} follows from $[a+b]^+\leq [a]^++[b]^+$, 
	Inequality~\eqref{viobandit:eq6} holds with probability at least $1-\delta$ employing Lemma~\ref{lem:conc2} and defining $\sigmavec_t\in \mathbb{R}^K$ such that $\sigma_t(a)=\nicefrac{C}{N_t(a)}$ for all $a\in[K]$, 
	Equation~\eqref{viobandit:eq7} follows from the fact that the quantities inside the $[\cdot]^+$ operator are positive,
	Inequality~\eqref{viobandit:eq8} holds with probability at least $1-2\delta$ employing the Azuma-Höeffding inequality after noticing that the confidence intervals can be capped to $1$ still making Lemma~\ref{lem:conc2} valid and a union bound and Inequality~\eqref{viobandit:eq9} holds since:
	\[ \sum_{t=1}^T\sum_{a\in[K]}\frac{\mathbb{I}_t(a)}{N_{t-1}(a)} \leq K(1+\ln(T))
	, \quad \sum_{t=1}^T\sum_{a\in[K]}\frac{\mathbb{I}_t(a)}{\sqrt{N_{t-1}(a)}} \leq 3\sqrt{KT},
	\]
	given that $\sum_{a\in[K]}\sqrt{N_T(a)}\leq \sqrt{K\sum_{a\in[K]}N_T(a)}\leq \sqrt{KT}$. 
	
	Combining the previous equations with a final union bound concludes the proof.
\end{proof}

\end{document}